\documentclass[11pt]{article}

\title{Benign Underfitting of Stochastic Gradient Descent} 

\author{
Tomer Koren%
\thanks{Blavatnik School of Computer Science, Tel Aviv University and Google Research; \texttt{tkoren@tauex.tau.ac.il}.}
\and
Roi Livni%
\thanks{Department of Electrical Engineering, Tel Aviv University; \texttt{rlivni@tauex.tau.ac.il}.}
\and
Yishay Mansour%
\thanks{Blavatnik School of Computer Science, Tel Aviv University and Google Research; \texttt{mansour.yishay@gmail.com}.}
\and
Uri Sherman%
\thanks{Blavatnik School of Computer Science, Tel Aviv University; \texttt{urisherman@mail.tau.ac.il}.}
}

\usepackage{header}
\usepackage{header_commands}

\usepackage{soul}

\begin{document}

\maketitle

\begin{abstract}
    We study to what extent may stochastic gradient descent (SGD) be understood as a ``conventional'' learning rule that achieves generalization performance by obtaining a good fit to training data.
    We consider the fundamental stochastic convex optimization framework, where (one pass, \emph{without}-replacement) SGD is classically known to minimize the population risk at rate $O(1/\sqrt n)$, and prove that, surprisingly, there exist problem instances where the SGD solution exhibits both empirical risk and generalization gap of $\Omega(1)$.
    Consequently, it turns out that SGD is not algorithmically stable in \emph{any} sense, and its generalization ability cannot be explained by uniform convergence or any other currently known generalization bound technique for that matter (other than that of its classical analysis).
    We then continue to analyze the closely related \emph{with}-replacement SGD, for which we show that an analogous phenomenon does not occur and prove that its population risk does in fact converge at the optimal rate.
    Finally, we interpret our main results in the context of without-replacement SGD for finite-sum convex optimization problems, and derive upper and lower bounds for the multi-epoch regime that significantly improve upon previously known results.
\end{abstract}

\section{Introduction}

Conventional wisdom in statistical learning revolves around what is traditionally known as the bias-variance dilemma; the classical theory stipulates the quality of fit to the training data be in a trade-off with model complexity, aiming for a sweet spot where training error is small but yet representative of performance on independent test data.

This perspective is reflected in the vast majority of generalization bound techniques offered by contemporary learning theory.
Uniform convergence approaches \citep{vapnik1971uniform, bartlett2002rademacher} seek capacity control over the model function class, and employ uniform laws of large numbers to argue convergence of sample averages to their respective expectations.
Algorithmic stability \citep{bousquet2002stability, shalev2010learnability} on the other hand, builds on controlling sensitivity of the learning algorithm to small changes in its input, and provides algorithm dependent bounds. 
Nevertheless, despite the conceptual and technical differences between these two methods, both ultimately produce risk bounds by controlling the training error, and the \emph{generalization gap}. The same is true for many other techniques, including sample compression~\citep{littlestone1986relating, arora2018stronger}, PAC-Bayes~\citep{mcallester1999pac,gintare2017nonvacuous}, and information theoretic generalization bounds~\citep{russo2016controlling, xu2017information, neu2021information}, to name a few.

In recent years it has become clear there are other, substantially different, ways to manage the fit vs.~complexity trade-off, that are in a sense incompatible with traditional generalization bound techniques. 
Evidently, heavily over-parameterized deep neural networks may be trained to perfectly fit training data and generalize well nonetheless \citep{zhang2017understanding, neyshabur2014search, neyshabur2019role}, thus seemingly disobeying conventional statistical wisdom.
This phenomenon has garnered significant attention, with a flurry of research works dedicated to developing new techniques that would be able to explain strong generalization performance of algorithms in this so called interpolation regime 
(see \citealp{bartlett2021deep, belkin2021fit} and references therein).
Notably, while these algorithms do not strike a balance between model complexity and fit to the data in the traditional sense, fundamentally, they still minimize the empirical risk as a proxy to test performance.

To summarize, in the classical and modern regimes alike, learning methods are thought of as minimizing some combination of the training error and generalization gap, with reasoning that relies in one way or another on the following trivial, yet arguably most profound, bound:
\begin{equation}\label{eq:intro} 
    \textsf{test-error} 
    ~\leq~ 
    \textsf{train-error} ~+~ |\textsf{generalization gap}|
    ~.
\end{equation}

In this work, we focus on stochastic gradient descent (SGD)---the canonical algorithm for training machine learning models nowadays---and ask whether its generalization performance can be understood through a similar lens.
We consider the fundamental stochastic convex optimization (SCO) framework, in which it is well known that SGD minimizes the population risk at a rate of $O(1/\sqrt n)~$\citep{nemirovskij1983problem}. Remarkably, the classical analysis targets the population risk directly, and in contrast with other generalization arguments, at least seemingly \emph{does not} rely on the above bound.
This highlights an intriguing question: Are these quantities, so fundamental to learning theory, relevant to the way that SGD ``works''? Put differently, is it possible to provide a more ``conventional" analysis of SGD that conforms with \eqref{eq:intro}?

Our main result shows that, perhaps surprisingly, there exist convex learning problems where the above bound becomes vacuous for SGD: namely, SGD minimizes the population risk, but at the same time, it \emph{does not} minimize the empirical risk and thus exhibits constant generalization gap. 
This accords neither with the traditional viewpoint nor with that of interpolation, as both recognize the empirical risk as the principal minimization objective.
We refer to this phenomenon as \emph{benign underfitting}: evidently, SGD underfits the training data,
but its classical analysis affirms this underfitting to be \emph{benign}, in the sense that test performance is never compromised as a result. 
Our construction presents a learning problem where the output of SGD with step size $\eta$ over $n$ i.i.d.~training examples is $\Omega(\eta \sqrt n)$ sub-optimal w.r.t.~the best fit possible, and consequently has a generalization gap of the same order. Notably, with the standard step size choice of $1/\sqrt n$ necessary to ensure the population risk converges at the optimal rate this lower bound amounts to a constant.

Many previously plausible explanations for generalization properties of this algorithm are thereby rendered inadequate, at least in the elementary convex setup we consider here.
First, it is clear that SGD cannot be framed as any reasonable regularized empirical risk minimization procedure for the simple reason that it does not minimize the empirical risk, which challenges the implicit regularization viewpoint to the generalization of SGD. Second, any attempt to explain generalization of SGD by uniform convergence over any (possibly data-dependent) hypotheses set cannot hold, simply because the sample average associated with the very same training set SGD was trained on is not necessarily close to its respective expectation. Finally, as it turns out, SGD provides for a strikingly natural example of an algorithm that generalizes well but is not stable in \emph{any} sense, as the most general notion of algorithmic stability is entirely equivalent to the generalization gap~\citep{shalev2010learnability}.

We then move on to study the generalization gap and empirical risk guarantees of SGD in a broader context. 
We study the case of non-convex and strongly convex component functions, and present natural extensions of our basic result.
In addition, we analyse the variant of SGD where datapoints are sampled with-replacement from the training set, in which case the train error is of course low but perhaps surprisingly the population risk is well behaved.
Finally, we make the natural connection to the study of without-replacement SGD for empirical risk minimization, and derive upper and lower bounds for the multi-epoch regime. These last two points are discussed in further detail in the following.

\paragraph{With vs without-replacement SGD.}
We may view one-pass SGD as processing the data via \emph{without}-replacement sampling from the training set, as randomly reshuffling the examples does not change their unconditional distribution. Thus, it is interesting to consider the generalization gap of the closely related algorithm given by 
running SGD over examples sampled \emph{with}-replacement from the training set. 
Considering instability (see \cref{sec:stability} for a detailed discussion) of SGD for non-smooth losses and the fact that this variant targets the empirical objective,
a priori it would seem this algorithm would overfit the training set and not provide strong population risk guarantees.
Surprisingly, our analysis presented in \cref{sec:wrsgd} reveals this is not the case, and that with a certain iterate averaging scheme the population risk converges at the optimal rate.
Consequently, it turns out the generalization gap is well bounded, and therefore that this variant constitutes a natural learning rule that is not stable in any sense but the most general one.

\paragraph{Without-replacement SGD for empirical risk minimization.}
The example featured in our main construction implies a lower bound of $\Omega(n^{-1/4})$ on the convergence rate of a single epoch of without-replacement SGD for finite sum optimization problems. In this setting, we have a set of $n$ convex losses and we wish to minimize their sum by running SGD over random shufflings of the losses.
While the smooth case has been studied extensively (e.g., \cite{recht2012toward, rajput2020closing, nagaraj2019sgd, safran2021random}), the non-smooth case has hardly received much attention. In \cref{sec:multiepoch} we extend our basic construction to a lower bound for the multi-epoch regime, and complement it with nearly matching upper bounds.

\paragraph{Our techniques.}
Fundamentally, we exploit the fact that dimension independent uniform convergence does not hold in SCO \citep{shalev2010learnability}. This is a prerequisite to any attempt at separating train and test losses of any hypothesis vector, let alone that produced by SGD.
Another essential condition is the instability of SGD for non-smooth losses, as any form of stability would immediately imply a generalization gap upper bound regardless of uniform convergence. 
Our main lower bound draws inspiration from constructions presented in the works of \cite{bassily2020stability} and \cite{amir21sgd}, both of which rely on instability, the latter also exploiting failure of uniform convergence.
However, neither of these contains the main ideas necessary to provoke the optimization dynamics required in our example.
A crucial ingredient in our construction consists of encoding into the SGD iterate information about previous training examples. This, combined with careful design of the loss function, gradient oracle and population distribution, allows correlating sub-gradients of independent training examples, and in turn guiding the SGD iterates to \emph{ascend} the empirical risk.

\subsection{Summary of main contributions}

To summarize, the main contributions of the paper are as follows:

\begin{itemize}[leftmargin=*,topsep=0pt,parsep=0pt]
    \item \textbf{One-pass SGD in SCO.} 
    In \cref{sec:main}, we study the basic SCO setup where the component losses are assumed to be individually convex, and present a construction where the expected empirical risk and therefore the generalization gap are both $\Omega(\eta \sqrt n)$. 
    We also provide extensions of our main construction demonstrating;
    \begin{itemize}[parsep=0pt]
        \item SCO with non-convex component functions may exhibit cases of benign \emph{overfitting}, where $\E\sb[b]{F(\hat w) - \hat F(\hat w)} = \Omega(\eta^2 n)$.
        \item In SCO with $\lambda$-strongly convex losses the worst case generalization gap is $\Omega(1/\lambda \sqrt n)$ for the standard step size choice. 
    \end{itemize}
    
    \item \textbf{With vs without replacement SGD in SCO.} 
    In \cref{sec:wrsgd}, we prove the variant of SGD where the training examples are processed via sampling \emph{with}-replacement from the \emph{training set} minimizes the population risk at the optimal rate, and thus enjoys a generalization gap upper bound bound of $O(1 / \sqrt n)$. 
        
    \item \textbf{Multi-epoch without-replacement SGD.}
    In \cref{sec:multiepoch}, we study convergence rates of without-replacement SGD for finite sum convex optimization problems.
    We prove a lower bound of $\Omega(n^{-1/4} K^{-3/4})$ on the optimization error after $K$ epochs over $n$ convex losses, and complement with upper bounds of $O(n^{-1/4} K^{-1/2})$ and $O(n^{-1/4} K^{-1/4})$ for respectively the multi-shuffle and single-shuffle SGD variants.
\end{itemize}

\subsection{Additional related work}

\paragraph{Gradient descent, algorithmic stability and generalization.}
Closely related to our work is the study of stability properties of SGD.
For smooth losses, \cite{hardt2016train} provide upper bounds on the generalization gap by appealing to uniform stability, yielding an $O(1/\sqrt n)$ rate for a single epoch of $n$ convex losses and the standard step size choice. 
In a later work, \cite{bassily2020stability} prove tight rates for uniform stability of SGD in the setting of \emph{non}-smooth losses, establishing these scale substantially worse; $\Theta(\eta \sqrt n)$ for step size $\eta$ and $n$ training examples.
Our work shows that in fact the worst case rate of the generalization gap completely coincides with the uniform stability rate of SGD.

A number of works prior to ours studied the extent to which SGD can be explained by implicit regularization in SCO.  \cite{kale2021sgd} study the setup where losses are smooth but only required to be convex in expectation, and show SGD may successfully learn when regularized ERM does not.
Prior to their work, \cite{dauber2020can} also rule out a wide range of implicit regularization based explanations of SGD in the basic SCO setup with convex losses.
On a more general level, our work is related to the study of stability and generalization in modern learning theory, pioneered by \cite{bousquet2002stability, shalev2010learnability}.
In particular, the failure of (dimension independent) uniform convergence in SCO was established in \cite{shalev2010learnability}.
The work of \cite{feldman2016generalization} improves the dimension dependence in the construction of \cite{shalev2010learnability} from exponential to linear in the number of training examples.
Notably, the construction featured in our main result requires the dimension to be exponential in the sample size, however the techniques of \cite{feldman2016generalization} do not readily extend to our setting. Thus, the optimal dimension dependence for a generalization gap lower bound is left for future work.

\paragraph{Without-replacement SGD for empirical risk minimization.}
A relatively long line of work studies convergence properties of without-replacement SGD from a pure optimization perspective (e.g., \cite{recht2012toward, nagaraj2019sgd, safran2020good, rajput2020closing, mishchenko2020random, safran2021random}).
Nearly all the papers in this line of work adopt the smoothness assumption, with near optimal bounds established by \cite{nagaraj2019sgd}.
An exception is the paper of \cite{shamir2016without} where an $O(1/\sqrt { n K })$ upper bound is obtained for $n$ datapoints and $K$ epochs, albeit only for generalized linear models over a bounded domain --- notably, a setting where uniform convergence holds.
Prior to this thread of research, \cite{nedic2001convergence} prove a convergence rate of $O(n/\sqrt K)$ for non-smooth loss functions that applies for \emph{any} ordering of the losses. To the best of our knowledge, this is also the state-of-the-art result for without-replacement SGD in the non-smooth setting without further assumptions on the loss functions.

\paragraph{Benign overfitting vs.~benign underfitting.}
While both benign underfitting and benign overfitting challenge traditional generalization techniques, that postulate the training error to represent the test error, as we discuss above these two phenomena point to very different regimes of learning. 
In particular, \cite{shamir2022implicit} shows that benign overfitting requires distributional assumptions for the interpolating algorithm to succeed. In contrast, we show that benign underfitting happens for SGD in a setting where it provably learns (namely, SCO), without any distributional assumptions. 
We also point out that \cref{cor:gap_no_underestimate} shows benign overfitting \emph{cannot} happen in the setup we consider, hence the two phenomena seem to rise in different setups.

\paragraph{Explaining generalization of interpolators.}
As already discussed, there is a large recent body of work dedicated to understanding why over-parameterized models trained by SGD to zero training error generalize well \citep[and references therein]{bartlett2021deep, belkin2021fit}. In particular, the work of \cite{bartlett2020benign} aims at explaining the phenomenon for high dimensional linear models. Some recent papers investigate limitations of certain techniques in explaining generalization of interpolating algorithms: 
\cite{nagarajan2019uniform} show uniform convergence fails to explain generalization of SGD in a setup where the generalization gap is in fact well bounded, thus in sharp contrast to our work;
\cite{bartlett2021failures} rule out the possibility of a large class of excess risk bounds to explain generalization of minimum norm interpolants. 
Unlike our work, they study properties of possible risk bounds when benign overfitting occurs, and thus do not pertain to SGD that never benignly overfits in SCO.

\section{Preliminaries}
We consider stochastic convex optimization (SCO) specified by a population distribution $\Z$ over a datapoint set $Z$, and loss function $f\colon W \times Z \to \R$ where $W\subset \R^d$ is convex and compact.
We denote
\begin{align*}
    F(w) &\eqq \E_{z \sim \Z} f(w; z),
    \tag{population loss}
    \\
    \hF(w) &\eqq \frac{1}{n}\sum_{i=1}^n f(w; z_i),
    \tag{empirical loss}
\end{align*}
where $\cb{z_1, \ldots, z_n} \subseteq Z$ stands for the training set, which we regularly denote by $S$. 
We let 
$w^\star \eqq \min_{w\in W} F(w)$
denote the population minimizer, and
$\wERM \eqq \min_{w\in W} \hF(w)$
denote the empirical risk minimizer (ERM).
The diameter of $W$ is defined by $\max_{x, y \in W}\cb{\norm{x - y}}$ where $\norm{\cdot}$ denotes the euclidean norm, and $\B^d_0(1) \eqq \cb[b]{x\in \R^d \mid \norm{x} \leq 1}$ denotes the $L_2$ unit ball in $\R^d$.
Given a training set $S = \cb{z_1, \ldots, z_n} \sim \Z^n$ and a learning algorithm that outputs a hypothesis $\hw$, we define the generalization gap to be the absolute value of the expected difference between test and train losses;
\begin{align} 
    \av{ \E_{S \sim \Z^n}\sb[b]{F(\hw) - \hF(\hw)} }.
    \tag{generalization gap}
\end{align}
% \begin{align}
%     F(\hw) - F(w^*)
%     =
%     \underbrace{F(\hw) - \hF(\hw)}_{\text{(signed) generalization gap}}
%     +
%     \underbrace{\hF(\hw) - \hF(\wERM)}_{\text{optimization error}}
%     +
%     \underbrace{\hF(\wERM) - F(w^\star)}_{\text{approximation error}}
%     .
%     \label{eq:risk_decom}
% \end{align}
% - risk / loss F(w)
% - error / excess risk F(w) - F(w^\star)
% - optimization error = train error
% - approximation error = \hF(\wERM) - F(w^\star)
Throughout most of the paper, we consider one-pass projected SGD over $S$;
\begin{align*}
    \text{initialize at } w_1 &\in W;
    \\
    \text{for } t=2, \ldots, n: \quad w_{t+1} &\gets 
        \Pi_W\b {w_t - \eta g_t}, \quad \text{with } 
    g_t \in \partial f(w_t; z_t),
\end{align*}
where $\partial f(w; z)$ denotes the set of sub-gradients of $f(\cdot; z) \to \R$ at the point $w\in W$, and $\Pi_W\colon \R^d \to W$ the projection operation onto $W$.

\section{A generalization gap lower bound for SGD}
\label{sec:main}
In this section, we establish our main result; that there exist convex learning problems where SGD incurs a large optimization error and therefore also a large generalization gap.
When losses are convex these two quantities are closely related since in expectation, the empirical risk minimizer cannot significantly outperform the population minimizer (a claim that will be made rigorous shortly after our main theorem).
Our construction builds on losses that are highly non-smooth, leading to SGD taking gradient steps that actually \emph{ascend} the empirical objective. 

\begin{theorem}\label{thm:lb_main}
    Let $n \in \N$, $n \geq 4$, $d \geq 2^{4 n\log n}$, and $W=\B_0^{2d}(1)$.
    Then there exists a distribution over instance set $Z$
    and a $4$-Lipschitz convex loss function $f\colon W \times Z \to \R$ such that running SGD initialized at $w_1=0$, with step size $\eta > 0$ over $S \sim \Z^n$ yields;
    \begin{enumerate}[label=(\roman*)]
        \item a large optimization error;
        \begin{aligni*}
            \E
            \sb{ \hF(\hw) - \hF(\wERM)} 
            = \Omega \b{ \min\cb{\eta \sqrt n, \frac{1}{\eta \sqrt n} } },
        \end{aligni*}
        \item a large generalization gap;
        \begin{aligni*}
            \E
            \sb{ \hF(\hw) - F(\hw)} 
            = \Omega \b{ \min\cb{\eta \sqrt n, \frac{1}{\eta \sqrt n} } },
        \end{aligni*}
    \end{enumerate}
    where $\hw$ is any suffix average of the iterates.
    In particular, for $\eta = \Theta(1/\sqrt n)$, the population risk is \begin{aligni*}
            \E%_{S \sim \Z^n}
            \sb{ F(\hw) - F(w^\star)} 
            = O(1/\sqrt n),
        \end{aligni*}
        while the generalization gap and training error are both
        \begin{aligni*}
            \Omega \b{1}.
        \end{aligni*}
\end{theorem}
A detailed proof of \cref{thm:lb_main} is deferred to \cref{sec:proof:lb_main}; in the following we provide an informal overview containing its principal ingredients. 

\begin{proof}[sketch]
    Let $Z \eqq \{0, 1\}^d$, and consider a population distribution $\Z$ such that $z(i)=1$ with probability $\delta$. We will use a loss function of the form 
    \begin{align*}
        f(w; z) \eqq \norm{z \odot w } + \phi(w; z),
    \end{align*}
    where $\odot$ denotes element-wise product. 
    The high level idea is that the norm component penalizes $w$'s that correlate with the given sample point $z$, and the $\phi$ function (the details of which are left for \cref{sec:proof:lb_main}) is tailored so that it drives the SGD iterates precisely to those areas in the $L_2$ ball where it correlates with the training set $\{z_1, \ldots, z_n\}$. In addition, the choice of parameters is such that the population loss is approximately zero over the entire domain.

    Taking $d$ sufficiently large compared to $\delta^{-1}$, we ensure that w.h.p., for every round $t\in[n]$ there exist many coordinates $i\in[d]$ with a prefix of ones; $z_1(i) = \cdots = z_{t-1} (i) = 1$ .  
    With $\delta$ chosen sufficiently small compared to $n$, we ensure that as long as $i\in[d]$ is any coordinate chosen independently of $\cb{z_{t+1}, \ldots, z_n}$, w.h.p.~this coordinate will have a suffix of zeros; $z_{t+1}(i) = \cdots = z_n (i) = 0$.
    
    Our goal is to make SGD take steps $w_{t+1} \approx w_t -\eta e_{i_t}$ (where $e_i$ denotes the $i$'th standard basis vector) where $i_t \in [d]$ is a coordinate with the aforementioned property of having a prefix of ones followed by a suffix of zeros. Note that since these steps are taken \emph{after} the prefix of ones has ended, they will inflict large empirical loss from the norm component, but will not be ``corrected'' by future steps owed to the suffix of zeros. To achieve this, we design $\phi$ so that it encodes the relevant information into the SGD iterates. Specifically, $\phi$ ``flags'' (using some extra dimensions) all coordinates $i\in[d]$ where a prefix of ones has been encountered. In addition, using another $\max$ component in $\phi$ we have that for all such coordinates $i$, $e_i \in \partial f(w_t; z)$ for any example $z$ (as this component in the loss depends only on the iterate $w_t$). In particular, we get that $e_i \in \partial f(w_t; z_t)$.
    Then, our gradient oracle just returns a subgradient pointing towards one of these coordinates (for convenience, we use the minimal one) which we denote by $i_t$, and SGD makes the desired step. 
    
    Notably, the coordinate $i_t$ chosen by the subgradient oracle is independent of future examples, and therefore will have a suffix of zeros w.h.p. Hence, as mentioned, this ensures no gradient signal after round $t$ will be able to correct the empirical risk ascent on $i_t$. Concluding, we have that for the final iterate $\hat w \eqq w_{n+1}$, we get
    $\hat w(i_t) = -\eta$ for all $t\in [n]$, therefore \begin{align*}
        \hF(\hat w) 
        = \frac{1}{n}\sum_{i=1}^n f(\hat w; z_i) 
        \approx \frac{1}{n}\sum_{i=1}^n 
            \norm{z_i \odot \hat w}
        \approx \norm{\hat w} \approx \sqrt{\eta^2 n} = \eta \sqrt n.
    \end{align*} 
    A similar argument requiring a few more technical steps shows the same is true for any suffix average $\hat w$.
    Noting that $\hF(0) = 0$, we get that the optimization error is $\Omega(\eta \sqrt n)$. 
    The implication for the generalization gap follows immediately with the standard step size choice of $\eta = 1/\sqrt n$, owed to SGD's population risk convergence guarantee. For an arbitrary step size, the result follows from a simple computation, and the proof is concluded.
\end{proof}

The magnitude of the generalization gap featured in \cref{thm:lb_main} stems from the large optimization error, which results in the empirical risk over-estimating the population risk by a large margin. Evidently, for convex losses the converse is always false; the empirical risk will never significantly under-estimate the population risk (a fact that will turn out false when losses are only required to be convex in expectation --- see \cref{sec:lb_nonconvex}).
Indeed, stability of the regularized ERM solution implies the ERM does not perform significantly better on the training set compared to the population minimizer $w^\star$.
\begin{lemma}\label{lem:star_aerm}
    Let $W \subset \R^d$ with diameter $D$, $\Z$ any distribution over $Z$, and  $f\colon W \times Z \to \R$ convex and $G$-Lipschitz in the first argument. Then
\begin{aligni*}
    \E\sb{ \hF(w^\star) - \hF(\wERM) } 
    \leq \frac{4 G D}{\sqrt n}
    .
\end{aligni*}
\end{lemma}
\begin{proof} % [of \cref{lem:star_aerm}]
    Denote the regularized ERM by
    \begin{aligni*}
      \hat w_{S}^\lambda \eqq \argmin_{w\in W} \cb{ 
        \frac{1}{n}\sum_{i=1}^n f_i(w; z_i) + \frac{\lambda}{2} \norm{w}^2 
      }.
    \end{aligni*}
    Observe,
    \begin{align*}
      F(w^\star)
      \leq \E F(\hat w_S^\lambda)
      \leq \E \hF (\hat w_S^\lambda) + \frac{4 G^2}{\lambda n}
      \leq \E \hF(\wERM) + \frac{\lambda}{2} D^2 + \frac{4 G^2}{\lambda n},
    %   \\
    %   &\leq \E \hF(\wERM) + \frac{4 G D}{\sqrt n},
    \end{align*}
    where the second inequality follows from stability of the regularized ERM (see \cref{lem:rerm_stability}).
    Choosing $\lambda \eqq 2G D/\sqrt n$, we get that
    \begin{align*}
        \E\sb{ \hF(w^\star) - \hF(\wERM) } 
        = F(w^\star) - \E \hF(\wERM) \leq \frac{4 G D}{\sqrt n},
    \end{align*}
    as claimed.
\end{proof}
Since the optimization error is always positive, we see that the upper bound given by \cref{lem:star_aerm} implies an upper bound on the difference between the population and empirical risks.
\begin{corollary}\label{cor:gap_no_underestimate}
  For any distribution $\Z$ over $Z$ and Lipschitz loss function $f\colon W \times Z \to \R$ convex in the first argument, running SGD with step size $\eta \eqq 1 / \sqrt n$ guarantees
  \begin{aligni*}
    \E \sb{ F(\hw) - \hF(\hw) } 
    \leq O(1/\sqrt n)
    .
  \end{aligni*}
\end{corollary}
\begin{proof}
    We have,
    \begin{align*}
        \E \sb[b]{ F(\hw) - \hF(\hw) }     
        &= \E \sb{ F(\hw) - F(w^\star) }
        + \E \sb[b]{ \hF(w^\star) - \hF(\hw) }
    \end{align*}
    The population error term on the RHS is $O(1/\sqrt n)$ by the classical analysis of SGD.
    The second term is bounded by \cref{lem:star_aerm}; 
    \begin{align*}
    \E\sb[b]{ \hF(w^\star) - \hF(\hw) } 
    \leq \E\sb[b]{ \hF(w^\star) - \hF(\wERM) } 
    \leq \fraci{4 G D}{\sqrt n},
    \end{align*}
    and the result follows.
\end{proof}
In the subsections that follow we continue to study the generalization gap in the context of common variants to the basic SCO setup.

\subsection{SCO with non-convex components}
\label{sec:lb_nonconvex}
When we relax the convexity assumption and only require the losses to be convex in expectation, we can construct a learning problem where SGD exhibits a case of benign \emph{overfitting}. In contrast to \cref{thm:lb_main}, here we actually drive the SGD iterates \emph{towards} an ERM solution, thus achieving a low optimization error and an empirical risk that under-estimates the population risk.
\begin{theorem}\label{thm:lb_non_convex}
    Let $n \in \N$, $n \geq 4$, $d \geq 2^{4 n \log n}$, $W = \B_0^{2 d}(1)$, and $\eta \leq 1/\sqrt n$.
    Then there exists a distribution $\Z$ over $Z$ and a $4$-Lipschitz loss $f\colon W \times Z \to \R$ where $\E_{z\sim \Z} f(w; z)$ is convex in $w$, such that for any suffix average $\hat w$ of SGD initialized at $w_1=0$, with step size $\eta$;
  \begin{align*}
      \E \sb{ F(\hw) - \hF(\hw)  } = \Omega(\eta^2 n).
  \end{align*}
\end{theorem}
The construction and proof of \cref{thm:lb_non_convex} given in \cref{sec:proof:lb_non_convex} follow a methodology similar to that of \cref{thm:lb_main}. 
Here however, we exploit non convex losses to form an empirical loss landscape where the ERM solution significantly outperforms the population minimizer $w^\star$ (notably, a feat not possible when losses are individually convex, by \cref{cor:gap_no_underestimate}).
Our loss function is defined by 
$f(w; z) \eqq \sum_{i=1}^d z(i)w(i)^2 + \phi(w; z)$, with each component playing a similar role as before. We work with the distribution $z \sim \cb{0, 1}^d$ where $z(i)=1$ 
w.p.~$\delta$, $z(i)=-1$ w.p.~$\delta$, and $z(i)=0$ w.p.~$1-2\delta$. The intuition is that coordinates accumulating many $-1$'s offer regions in the $L_2$ ball where the empirical risk is ``too good'' compared to the population risk. We tailor the extra dimensions and $\phi$ in coordination with the $-1$ values so that the sub-gradients guide the SGD iterates towards these regions, in exactly the same manner the construction of \cref{thm:lb_main} drives the iterates to high loss regions. We note that while the statement of \cref{thm:lb_non_convex} is specialized to step size smaller than $1/\sqrt n$, it may be extended to any step size using arguments similar to those given in the proof of \cref{thm:lb_main}. 

\subsection{SCO with strongly convex components}
Our basic construction extends to the strongly convex case by making only technical modification to \cref{thm:lb_main}.
The theorem below concerns the standard step size choice for strongly convex objectives. We provide its proof in \cref{sec:proof:sc_lb_main}.
\begin{theorem}\label{thm:sc_lb_main}
    Let $n \in \N$, $n \geq 10$, $d \geq 2^{4 n\log n}$, $W=\B_0^{2d}(1)$, and $\lambda \geq \fraci{1}{\sqrt n}$.
    Then there exists a distribution over instance set $Z$
    and a $4$-Lipschitz, $\lambda$-strongly convex loss function $f\colon W \times Z \to \R$ 
    \begin{enumerate}[label=(\roman*),nosep]
        \item the optimization error is large;
        \begin{aligni*}
            \E_{S \sim \Z^n}
            \sb[b]{ \hF(\hw) - \hF(\wERM)} 
            = \Omega \b{ \frac{1}{\lambda \sqrt n} },
        \end{aligni*}
        \item the generalization gap is large;
        \begin{aligni*}
            \E_{S \sim \Z^n}
            \sb[b]{ \hF(\hw) - F(\hw)} 
            = \Omega \b{ \frac{1}{\lambda \sqrt n} },
        \end{aligni*}
    \end{enumerate}
    where $\hw$ is any suffix average of SGD  initialized at $w_1=0$, with step size schedule $\eta_t=1/\lambda t$. Furthermore, the problem instance where this occurs is precisely the $\lambda$ regularized version of the example featured in \cref{thm:lb_main}.
\end{theorem}
We note that an immediate implication of the above theorem is that if we seek a generalization gap upper bound for a weakly convex problem by means of regularization (meaning, by running SGD on a regularized problem), we would have to take $\lambda \geq 1$ to guarantee a gap of $O(1/\sqrt n)$. To see this, note that the generalization gap (of any hypothesis) of the regularized problem is the same as that of the original.
On the other hand, taking $\lambda \geq 1$ will of course be detrimental to the population error guarantee. Hence, one cannot circumvent the generalization gap lower bound by regularization without compromising the population error.

We conclude this section with a note regarding stability rates of SGD in non-smooth SCO. Implicit in \cref{thm:lb_main}, is that average stability of SGD coincides with the tight uniform stability rate of $\Theta(\eta \sqrt n)$ established by \cite{bassily2020stability}. This is because \cref{thm:lb_main} provides the $\Omega(\eta \sqrt n)$ lower bound on the most general stability notion, which is precisely the generalization gap \citep{shalev2010learnability}.
We refer the reader to \cref{sec:stability} for a more elaborate discussion.

\section{SGD with vs without replacement}
\label{sec:wrsgd}
In this section, we consider a different algorithm in the context of the basic SCO setup; SGD over examples drawn \emph{with}-replacement from the \emph{training set}.
This is not to be confused with one-pass SGD discussed in \cref{sec:main}, which corresponds to without-replacement SGD on the training set, or alternatively with-replacement SGD over the \emph{population} distribution.
Given a training set $S = \cb{z_1, \ldots, z_n} \sim \Z^n$, we define with-replacement projected SGD initialized at $w_1 \in W$ by
\begin{align*}
    w_{t+1} &\gets 
        \Pi_W\b {w_t - \eta \hat g_t}, \quad \text{where } 
    \hat g_t \in \partial f(w_t; \hat z_t)
    \text{ and } \hat z_t \sim \Unif(S).
\end{align*}
Perhaps surprisingly, this version of SGD does not overfit the training data; our theorem below establishes that with proper iterate averaging, the population risk converges at the optimal rate. 

\begin{theorem}\label{thm:wrsgd_population_ub}
    Let $W\subset \R^d$ with diameter $D$, $\Z$ be any distribution over $Z$, and 
    $f: W \times Z \to \R$ be convex and $G$-Lipschitz in the first argument.
    Let $S \sim \Z^n$ be a training set of $n \in \N$ datapoints drawn i.i.d.~from $\Z$, and consider running SGD over training examples sampled with-replacement, uniformly and independently from $S$.
    Then, for step size $\eta=\frac{D}{G\sqrt n}$ and $\wbar w \eqq \frac{2}{n+1} \sum_{t=1}^n \frac{n-t+1}{n} w_t$, the following upper bound holds;
    \begin{align*}
        \E\sb{F(\wbar w) - F(w^\star)}
        \leq \frac{10 G D}{\sqrt n}
        .
    \end{align*}
  \end{theorem}
  \begin{proof} % [of \cref{thm:wrsgd_population_ub}]
    Fix a time-step $t\in[n]$, and observe that if we don't condition on $S$, we may view the random datapoint $\hat z_t$ as a mixture between a fresh i.i.d.~sample from the population and a uniformly distributed sample from the previously processed datapoints
    $\hat S_{t-1} \eqq \cb{\hat z_1, \ldots, \hat z_{t-1}}$;
    \begin{align*}
        \hat z_t \mid \hat S_{t-1}
        = \begin{cases}
            z \sim \Z \quad &\text{w.p. } 1 - \frac{t-1}{n},
            \\
            z \sim \Unif(\hat S_{t-1}) \quad &\text{w.p. } \frac{t-1}{n}.
        \end{cases}
    \end{align*}
    With this in mind, denote $\hat f_t(w) \eqq f(w; \hat z_t)$, fix $\hat S_{t-1}$ and observe:
    \begin{align*}
        \E_{\hat z_t}\sb{\hat f_t(w_t) - \hat f_t(w^\star) \mid \hat S_{t-1}}
        &=\b[B]{1 - \frac{t-1}{n}} \E_{z\sim \Z} \sb{ f(w_t; z) - f(w^\star; z) }
        \\
        &+\frac{t-1}{n} \frac{1}{t-1} \sum_{i=1}^{t-1} \hat f_i(w_t) - \hat f_i(w^\star).
    \end{align*}
    Rearranging and taking expectation with respect to $\hat S_{t-1}$ we obtain
    \begin{align}
        \b[B]{1 - \frac{t-1}{n}} \E \sb{ f(w_t; z) - f(w^\star; z) }
        &= \E\sb{\hat f_t(w_t) - \hat f_t(w^\star)}
        + \E \sb{ \frac{1}{n} \sum_{i=1}^{t-1} \hat f_i(w^\star) - \hat f_i(w_t) }
        \nonumber \\
        &\leq \E\sb{\hat f_t(w_t) - \hat f_t(w^\star)}
        + \frac{4 G D \sqrt t}{n}
        \label{eq:iidtrain_1},
    \end{align}
    where the inequality follows from \cref{lem:star_aerm}.
    Now, by a direct computation we have
    $\sum_{t=1}^n \b[B]{1 - \frac{t-1}{n}} = \frac{n + 1}{2}$,
    which motivates setting 
    $\wbar w \eqq \frac{2}{n+1} \sum_{t=1}^n \frac{n-t+1}{n} w_t$.
    % \begin{align*}
    %     \sum_{t=1}^n \b[B]{1 - \frac{t-1}{n}} 
    %     = n - \frac{1}{n}\sum_{t=0}^{n-1} t
    %     = n - \frac{1}{2 n}(n-1)n
    %     = \frac{n + 1}{2},
    % \end{align*}
    By convexity of $F$, \cref{eq:iidtrain_1}, and the standard regret analysis of gradient descent \citep[e.g., ][]{hazan2019introduction} we now have
    \begin{align*}
        \E\sb{F(\wbar w) - F(w^\star)}
        &\leq \frac{2}{n+1} \sum_{t=1}^n 
            \b[B]{1 - \frac{t-1}{n}} \E\sb{F(w_t) - F(w^\star)}
        \\
        &\leq \frac{2}{n+1} \sum_{t=1}^n 
            \E\sb{\hat f_t(w_t) - \hat f_t(w^\star)}
            + \frac{2}{n+1} \sum_{t=1}^n \frac{4 G D \sqrt t}{n}
        \\
        &\leq \frac{2}{n} \E\sb{\sum_{t=1}^n 
            \hat f_t(w_t) - \hat f_t(w^\star)}
            + \frac{8 G D }{\sqrt n}
        \\
        &\leq \frac{2}{n} \b{ \frac{D^2}{2 \eta} + \frac{\eta G^2}{2} }
            + \frac{8 G D }{\sqrt n}
        \\
        &= \frac{10 G D}{\sqrt n},
    \end{align*}
    where the last inequality follows by our choice of $\eta=\frac{D}{G\sqrt n}$. 
  \end{proof}
  Evidently, the averaging scheme dictated by \cref{thm:wrsgd_population_ub} does little to hurt the empirical risk convergence guarantee, which follows from the standard analysis with little modifications (for completeness we provide a formal statement and proof in \cref{sec:proof:wrsgd_std_ub}). Combined with \cref{lem:star_aerm}, this immediately implies a generalization gap upper bound for with-replacement SGD.
  Notably, this shows with-replacement SGD provides for an example of a (natural) algorithm in the SCO learning setup that is not even stable on-average, but nonetheless has a well bounded generalization gap. 
  We refer the reader to the discussion in \cref{sec:stability} for more details.
  
\begin{corollary}\label{cor:wrsgd_gengap}
  For any distribution $\Z$ and loss function $f\colon W \times Z \to \R$ convex and Lipschitz in the first argument, running SGD with step size and averaging as specified in \cref{thm:wrsgd_population_ub} ensures
  \begin{align*}
    \av[b]{ \E \sb[b]{ F(\wbar w) - \hF(\wbar w) } }
    \leq O(1/\sqrt n)
    .
  \end{align*}
\end{corollary}
\begin{proof}
    We have; 
    \begin{align*}
        \av[b]{\E \sb[b]{ F(\wbar w) - \hF(\wbar w) }}
        \leq 
            \av{ \E \sb{ F(\wbar w) - F(w^\star) } }
            + \av[b]{ \E \sb[b]{ \hF(w^\star) - \hF(\wERM)} }
            + \av[b]{ \E \sb[b]{ \hF(\wERM) - \hF(\wbar w)} }
        % \leq \E \sb{ F(\hat w) - F(w^\star) } + \frac{4 G D}{\sqrt n}
        .
    \end{align*}
    The first term is upper bounded by convergence of the population risk provided by \cref{thm:wrsgd_population_ub}, the second by \cref{lem:star_aerm}, and the third by the standard analysis of SGD (see \cref{sec:proof:wrsgd_std_ub}).
\end{proof}

\section{Multi-epoch SGD for empirical risk minimization}
\label{sec:multiepoch}
In this section, we forgo the existence of a population distribution and discuss convergence properties of without-replacement SGD (wor-SGD) for finite sum optimization problems.
A relatively long line of work discussed in the introduction studies this problem in the \emph{smooth} case.
The work of \cite{nagaraj2019sgd} noted smoothness is a necessary assumption to obtain rates that are strictly better than the $O(1/\sqrt {n K})$ guaranteed by with-replacement SGD for $n$ losses and $K$ epochs, due to a lower bound that follows from the deterministic case (e.g., \cite{bubeck2015convex}). Here we establish that smoothness is in fact necessary to obtain rates that are not \emph{strictly worse} than with-replacement SGD.
We consider running multiple passes of wor-SGD to solve the finite sum optimization problem given by the objective
\begin{align}
    F(w) \eqq \frac{1}{n} \sum_{t=1}^n f(w; t)
    \label{eq:worsgd_objective}
\end{align}
where $\cb{f(w; t)}_{t=1}^n$ is a set of $n$ convex, $G$-Lipschitz losses defined over a convex and compact domain $W \subseteq \R^d$. 
Throughout this section we let $w^\star \eqq \min_{w\in W} F(w)$ denote the minimizer of the objective \cref{eq:worsgd_objective}. 
In every epoch $k\in [K]$ we process the losses in the order specified by a permutation $\pi_k: [n] \leftrightarrow [n]$ sampled uniformly at random, either once in the beginning of the algorithm (single-shuffle), or at the onset of every epoch (multi-shuffle). Multi-epoch wor-SGD initialized at $w_1^1 \in W$ is specified by the following equations;
\begin{align*}
    w_{t+1}^k &\gets \Pi_W ( w_t^k - \eta g_t^k),
    \;\; \text{where } g_t^k \in \partial f_t^k(w_t^k)
    \\
    w_1^{k+1} &\eqq w_{n+1}^k,
\end{align*}
where we denote $f_t^k(w) \eqq f(w; \pi_k(t))$.
A near-immediate implication of \cref{thm:lb_main} is that there exists a set of convex losses on which a single epoch of wor-SGD cannot converge at a rate faster than $1/n^{1/4}$.
\cref{thm:melb_main} presented below extends our basic construction from \cref{thm:lb_main} to accommodate multiple epochs. The main challenge here is in devising a mechanism that will allow fresh bad gradient steps to take place on every new epoch.
\begin{theorem}\label{thm:melb_main}
    Let $n, K \in \N$, $K \geq 4, n \geq 4$, $c\eqq 4/(2^{1/K} - 1)$,
    $d \geq 2^{6n \log(c n K)} $, and $W=\B_0^{d'}(1)$ where 
    $d' = (n K + 1)d$.
    Then there exists a set of $n$ convex, $4$-Lipschitz losses such that after $K$ epochs of either multi-shuffle or single-shuffle SGD initialized at $w_1^1=0$ with step size $\eta \leq 1/\sqrt{2 n K}$, it holds that 
  \begin{align*}
      \E \sb{ F(\hat w) - F(w^*) }
      = \Omega\b{\min\cb{1, \eta \sqrt{\frac{n}{J}}   + \frac{1}{\eta n K} + \eta}},
  \end{align*}
  where $\hat w$ is any suffix average of the last $J$ epochs.
  In particular, we obtain a bound of $\Omega\b{n^{-1/4}K^{-3/4}}$ 
  for any suffix average and any choice of $\eta$.
\end{theorem}
The proof of \cref{thm:melb_main} is provided in \cref{sec:proof:melb_main}. 
The construction in the proof takes the idea that the training set can be encoded in the SGD iterate to the extreme. The loss function and gradient oracle are designed in such a way so as to record the training examples in their full form and order into the iterate. We then exploit this encoded information with an ``adversarial'' gradient oracle that returns the bad sub-gradients on each gradient step in every new epoch.

Next, we complement \cref{thm:melb_main} with an upper bound that builds on stability arguments similar to those of the smooth case \citep{nagaraj2019sgd}. Importantly though, lack of smoothness means worse stability rates and necessitates extra care in the technical arguments. Below, we prove the multi-shuffle case, and defer the full details for the single-shuffle case to \cref{sec:proof:mesgd_ub}. 
  \begin{theorem}\label{thm:mesgd_ub}
    Let $S = \cb{f(w; t)}_{t=1}^n$ be a set of $n$ convex, $G$-Lipschitz losses over a convex and compact domain $W \subseteq \R^d$ of diameter $D$, and consider running $K\geq 1$ epochs of wor-SGD over $S$. Then, 
    we have the following guarantees:
    \begin{enumerate}[label=(\roman*),nosep]
        \item For multi-shuffle, with step-size $\eta = D/ (G n^{3/4}K^{1/2} )$, we have
        \begin{align*}
            \E \sb{ F(\hat w) - F(w^\star) }
            \leq \frac{3 G D}{n^{\fraci{1}{4}} K^{1/2}}.
        \end{align*}
        
        \item For single-shuffle, with step-size $\eta = D/ (2 G n^{3/4}K^{3/4} )$ and assuming $K \geq n$, we have
        \begin{align*}
            \E \sb{ F(\hat w) - F(w^\star) }
            \leq \frac{10 G D}{n^{1/4} K^{1/4}}.
        \end{align*}

    \end{enumerate}
    In both of the above bounds, $\hat w = \frac{1}{n K}\sum_{k\in[K], t\in[n]} w_t^k$, and the expectation is over the random permutations of losses.
  \end{theorem}
  
    \begin{proof}[ (multi-shuffle case)]
    Observe;
    \begin{align*}
        F(\hat w) - F(w^\star)
        &\leq \frac{1}{n K}\sum_{k=1}^K \sum_{t=1}^n F(w_t^k) - F(w^\star)
        \\
        &= \frac{1}{n K}\sum_{k=1}^K \sum_{t=1}^n F(w_t^k) - f_t^k(w^\star)
        \\
        &= \frac{1}{n K}\sum_{k=1}^K \sum_{t=1}^n F(w_t^k) - f_t^k(w_t^k)
        + \frac{1}{n K}\sum_{k=1}^K \sum_{t=1}^n f_t^k(w_t^k) - f_t^k(w^\star)
        \\
        &\leq
        \frac{1}{n K}\sum_{k=1}^K \sum_{t=1}^n F(w_t^k) - f_t^k(w_t^k)
        + \frac{D^2}{ 2 \eta n K} + \frac{\eta G^2}{2},
    \end{align*}
    with the last inequality following from the standard $n K$ round regret bound for gradient descent \citep[see e.g.,][]{hazan2019introduction}. 
    To bound the other term, using \cref{lem:stab_gen}, we relate the difference between the without-replacement loss distribution and the full batch objective to the uniform stability rate of SGD, which may then be bounded by applying \cref{lem:sgd_stab}:
    \begin{align*}
        \E \sb{ F(w_t^k) - f_t^k(w_t^k)) }
        &=
        \E_{\pi_1, \ldots, \pi_{k-1}} \E_{\pi_k} \sb{ F(w_t^k) - f_t^k(w_t^k) \mid w_1^k }
        \\
        &\leq 
        \E_{\pi_1, \ldots, \pi_{k-1}} \sb{ G \epstabSGD(t-1)}
        \\
        &= G \epstabSGD(t-1)
        \\
        &\leq 2 \eta G^2 \sqrt t.
    \end{align*}
    Concluding, we have that
    \begin{align*}
        \E \sb{ F(\hat w) - F(w^\star) }
        &\leq
        \frac{1}{n K}\sum_{k=1}^K \sum_{t=1}^n\E \sb{  F(w_t^k) - f_t^k(w_t^k) }
        + \frac{D^2}{ 2 \eta n K} + \frac{\eta G^2}{2}
        \\
        &\leq
        \frac{2}{n K}\sum_{k=1}^K \sum_{t=1}^n\eta G^2 \sqrt t
        + \frac{D^2}{ 2 \eta n K} + \frac{\eta G^2}{2}
        \\
        &\leq
        2 \eta G^2 \sqrt n
        + \frac{D^2}{ 2 \eta n K} + \frac{\eta G^2}{2}
        \\
        &\leq \frac{3 G D}{n^{1/4} K^{1/2}}
        ,
    \end{align*}
    where the last inequality follows from our choice of $\eta = D/ (G n^{3/4}K^{1/2} )$.

\end{proof}

\subsection*{Acknowledgements and funding disclosure}
This work was supported by the European Research Council (ERC) under the European Union’s Horizon 2020 research and innovation program (grant agreement No.~882396), by the Israel Science Foundation (grants number 993/17, 2549/19, 2188/20), by the Len Blavatnik and the Blavatnik Family foundation, by the Yandex Initiative in Machine Learning at Tel Aviv University, 
by a grant from the Tel Aviv University Center for AI and Data Science (TAD),
and by an unrestricted gift from Google. Any opinions, findings, and conclusions or recommendations expressed in this work are those of the author(s) and do not necessarily reflect the views of Google.

% \newpage

\bibliography{main}

\begin{thebibliography}{38}
\providecommand{\natexlab}[1]{#1}
\providecommand{\url}[1]{\texttt{#1}}
\expandafter\ifx\csname urlstyle\endcsname\relax
  \providecommand{\doi}[1]{doi: #1}\else
  \providecommand{\doi}{doi: \begingroup \urlstyle{rm}\Url}\fi

\bibitem[Amir et~al.(2021)Amir, Koren, and Livni]{amir21sgd}
I.~Amir, T.~Koren, and R.~Livni.
\newblock {SGD} generalizes better than {GD} (and regularization doesn't help).
\newblock In \emph{Conference on Learning Theory, {COLT} 2021}, volume 134 of
  \emph{Proceedings of Machine Learning Research}, pages 63--92. {PMLR}, 2021.

\bibitem[Arora et~al.(2018)Arora, Ge, Neyshabur, and Zhang]{arora2018stronger}
S.~Arora, R.~Ge, B.~Neyshabur, and Y.~Zhang.
\newblock Stronger generalization bounds for deep nets via a compression
  approach.
\newblock In \emph{International Conference on Machine Learning}, pages
  254--263. PMLR, 2018.

\bibitem[Bartlett and Long(2021)]{bartlett2021failures}
P.~L. Bartlett and P.~M. Long.
\newblock Failures of model-dependent generalization bounds for least-norm
  interpolation.
\newblock \emph{Journal of Machine Learning Research}, 22\penalty0
  (204):\penalty0 1--15, 2021.

\bibitem[Bartlett and Mendelson(2002)]{bartlett2002rademacher}
P.~L. Bartlett and S.~Mendelson.
\newblock Rademacher and gaussian complexities: Risk bounds and structural
  results.
\newblock \emph{Journal of Machine Learning Research}, 3\penalty0
  (Nov):\penalty0 463--482, 2002.

\bibitem[Bartlett et~al.(2020)Bartlett, Long, Lugosi, and
  Tsigler]{bartlett2020benign}
P.~L. Bartlett, P.~M. Long, G.~Lugosi, and A.~Tsigler.
\newblock Benign overfitting in linear regression.
\newblock \emph{Proceedings of the National Academy of Sciences}, 117\penalty0
  (48):\penalty0 30063--30070, 2020.

\bibitem[Bartlett et~al.(2021)Bartlett, Montanari, and
  Rakhlin]{bartlett2021deep}
P.~L. Bartlett, A.~Montanari, and A.~Rakhlin.
\newblock Deep learning: a statistical viewpoint.
\newblock \emph{arXiv preprint arXiv:2103.09177}, 2021.

\bibitem[Bassily et~al.(2020)Bassily, Feldman, Guzm{\'a}n, and
  Talwar]{bassily2020stability}
R.~Bassily, V.~Feldman, C.~Guzm{\'a}n, and K.~Talwar.
\newblock Stability of stochastic gradient descent on nonsmooth convex losses.
\newblock \emph{Advances in Neural Information Processing Systems}, 33, 2020.

\bibitem[Belkin(2021)]{belkin2021fit}
M.~Belkin.
\newblock Fit without fear: remarkable mathematical phenomena of deep learning
  through the prism of interpolation.
\newblock \emph{arXiv preprint arXiv:2105.14368}, 2021.

\bibitem[Bousquet and Elisseeff(2002)]{bousquet2002stability}
O.~Bousquet and A.~Elisseeff.
\newblock Stability and generalization.
\newblock \emph{The Journal of Machine Learning Research}, 2:\penalty0
  499--526, 2002.

\bibitem[Bubeck(2015)]{bubeck2015convex}
S.~Bubeck.
\newblock Convex optimization: Algorithms and complexity.
\newblock \emph{Foundations and Trends{\textregistered} in Machine Learning},
  8\penalty0 (3-4):\penalty0 231--357, 2015.

\bibitem[Dauber et~al.(2020)Dauber, Feder, Koren, and Livni]{dauber2020can}
A.~Dauber, M.~Feder, T.~Koren, and R.~Livni.
\newblock Can implicit bias explain generalization? stochastic convex
  optimization as a case study.
\newblock \emph{Advances in Neural Information Processing Systems}, 33, 2020.

\bibitem[Dziugaite and Roy(2018)]{gintare2017nonvacuous}
G.~K. Dziugaite and D.~M. Roy.
\newblock Computing nonvacuous generalization bounds for deep (stochastic)
  neural networks with many more parameters than training data.
\newblock In \emph{Thirty-Third Conference on Uncertainty in Artificial
  Intelligence, UAI 2017}, 2018.

\bibitem[Feldman(2016)]{feldman2016generalization}
V.~Feldman.
\newblock Generalization of erm in stochastic convex optimization: The
  dimension strikes back.
\newblock In \emph{Advances in Neural Information Processing Systems},
  volume~29, 2016.

\bibitem[Hardt et~al.(2016)Hardt, Recht, and Singer]{hardt2016train}
M.~Hardt, B.~Recht, and Y.~Singer.
\newblock Train faster, generalize better: Stability of stochastic gradient
  descent.
\newblock In \emph{International Conference on Machine Learning}, pages
  1225--1234. PMLR, 2016.

\bibitem[Hazan(2019)]{hazan2019introduction}
E.~Hazan.
\newblock Introduction to online convex optimization.
\newblock \emph{arXiv preprint arXiv:1909.05207}, 2019.

\bibitem[Kale et~al.(2021)Kale, Sekhari, and Sridharan]{kale2021sgd}
S.~Kale, A.~Sekhari, and K.~Sridharan.
\newblock Sgd: The role of implicit regularization, batch-size and
  multiple-epochs.
\newblock \emph{arXiv preprint arXiv:2107.05074}, 2021.

\bibitem[Littlestone and Warmuth(1986)]{littlestone1986relating}
N.~Littlestone and M.~Warmuth.
\newblock Relating data compression and learnability, 1986.

\bibitem[McAllester(1999)]{mcallester1999pac}
D.~A. McAllester.
\newblock Pac-bayesian model averaging.
\newblock In \emph{Proceedings of the twelfth annual conference on
  Computational learning theory}, pages 164--170, 1999.

\bibitem[Mishchenko et~al.(2020)Mishchenko, Khaled Ragab~Bayoumi, and
  Richt{\'a}rik]{mishchenko2020random}
K.~Mishchenko, A.~Khaled Ragab~Bayoumi, and P.~Richt{\'a}rik.
\newblock Random reshuffling: Simple analysis with vast improvements.
\newblock \emph{Advances in Neural Information Processing Systems}, 33, 2020.

\bibitem[Nagaraj et~al.(2019)Nagaraj, Jain, and Netrapalli]{nagaraj2019sgd}
D.~Nagaraj, P.~Jain, and P.~Netrapalli.
\newblock Sgd without replacement: Sharper rates for general smooth convex
  functions.
\newblock In \emph{International Conference on Machine Learning}, pages
  4703--4711. PMLR, 2019.

\bibitem[Nagarajan and Kolter(2019)]{nagarajan2019uniform}
V.~Nagarajan and J.~Z. Kolter.
\newblock Uniform convergence may be unable to explain generalization in deep
  learning.
\newblock \emph{Advances in Neural Information Processing Systems}, 32, 2019.

\bibitem[Nedi{\'c} and Bertsekas(2001)]{nedic2001convergence}
A.~Nedi{\'c} and D.~Bertsekas.
\newblock Convergence rate of incremental subgradient algorithms.
\newblock In \emph{Stochastic optimization: algorithms and applications}, pages
  223--264. Springer, 2001.

\bibitem[Nemirovskij and Yudin(1983)]{nemirovskij1983problem}
A.~S. Nemirovskij and D.~B. Yudin.
\newblock Problem complexity and method efficiency in optimization, 1983.

\bibitem[Neu(2021)]{neu2021information}
G.~Neu.
\newblock Information-theoretic generalization bounds for stochastic gradient
  descent.
\newblock \emph{arXiv preprint arXiv:2102.00931}, 2021.

\bibitem[Neyshabur et~al.(2014)Neyshabur, Tomioka, and
  Srebro]{neyshabur2014search}
B.~Neyshabur, R.~Tomioka, and N.~Srebro.
\newblock In search of the real inductive bias: On the role of implicit
  regularization in deep learning.
\newblock \emph{arXiv preprint arXiv:1412.6614}, 2014.

\bibitem[Neyshabur et~al.(2019)Neyshabur, Li, Bhojanapalli, LeCun, and
  Srebro]{neyshabur2019role}
B.~Neyshabur, Z.~Li, S.~Bhojanapalli, Y.~LeCun, and N.~Srebro.
\newblock The role of over-parametrization in generalization of neural
  networks.
\newblock In \emph{7th International Conference on Learning Representations,
  {ICLR} 2019, New Orleans, LA, USA, May 6-9, 2019}. OpenReview.net, 2019.

\bibitem[Rajput et~al.(2020)Rajput, Gupta, and
  Papailiopoulos]{rajput2020closing}
S.~Rajput, A.~Gupta, and D.~Papailiopoulos.
\newblock Closing the convergence gap of {SGD} without replacement.
\newblock In \emph{International Conference on Machine Learning}, pages
  7964--7973. PMLR, 2020.

\bibitem[Recht and R{\'e}(2012)]{recht2012toward}
B.~Recht and C.~R{\'e}.
\newblock Toward a noncommutative arithmetic-geometric mean inequality:
  conjectures, case-studies, and consequences.
\newblock In \emph{Conference on Learning Theory}, pages 11--1. JMLR Workshop
  and Conference Proceedings, 2012.

\bibitem[Russo and Zou(2016)]{russo2016controlling}
D.~Russo and J.~Zou.
\newblock Controlling bias in adaptive data analysis using information theory.
\newblock In \emph{Artificial Intelligence and Statistics}, pages 1232--1240.
  PMLR, 2016.

\bibitem[Safran and Shamir(2020)]{safran2020good}
I.~Safran and O.~Shamir.
\newblock How good is {SGD} with random shuffling?
\newblock In \emph{Conference on Learning Theory}, pages 3250--3284. PMLR,
  2020.

\bibitem[Safran and Shamir(2021)]{safran2021random}
I.~Safran and O.~Shamir.
\newblock Random shuffling beats sgd only after many epochs on ill-conditioned
  problems.
\newblock \emph{arXiv preprint arXiv:2106.06880}, 2021.

\bibitem[Shalev-Shwartz et~al.(2010)Shalev-Shwartz, Shamir, Srebro, and
  Sridharan]{shalev2010learnability}
S.~Shalev-Shwartz, O.~Shamir, N.~Srebro, and K.~Sridharan.
\newblock Learnability, stability and uniform convergence.
\newblock \emph{The Journal of Machine Learning Research}, 11:\penalty0
  2635--2670, 2010.

\bibitem[Shamir(2016)]{shamir2016without}
O.~Shamir.
\newblock Without-replacement sampling for stochastic gradient methods.
\newblock \emph{Advances in neural information processing systems},
  29:\penalty0 46--54, 2016.

\bibitem[Shamir(2022)]{shamir2022implicit}
O.~Shamir.
\newblock The implicit bias of benign overfitting.
\newblock \emph{arXiv preprint arXiv:2201.11489}, 2022.

\bibitem[Sherman et~al.(2021)Sherman, Koren, and Mansour]{sherman2021optimal}
U.~Sherman, T.~Koren, and Y.~Mansour.
\newblock Optimal rates for random order online optimization.
\newblock \emph{Advances in Neural Information Processing Systems}, 34, 2021.

\bibitem[Vapnik(1971)]{vapnik1971uniform}
V.~Vapnik.
\newblock On the uniform convergence of relative frequencies of events to their
  probabilities.
\newblock \emph{Theory of Probability and its Applications}, 16\penalty0
  (2):\penalty0 264--281, 1971.

\bibitem[Xu and Raginsky(2017)]{xu2017information}
A.~Xu and M.~Raginsky.
\newblock Information-theoretic analysis of generalization capability of
  learning algorithms.
\newblock In \emph{Proceedings of the 31st International Conference on Neural
  Information Processing Systems}, pages 2521--2530, 2017.

\bibitem[Zhang et~al.(2017)Zhang, Bengio, Hardt, Recht, and
  Vinyals]{zhang2017understanding}
C.~Zhang, S.~Bengio, M.~Hardt, B.~Recht, and O.~Vinyals.
\newblock Understanding deep learning requires rethinking generalization.
\newblock In \emph{5th International Conference on Learning Representations,
  {ICLR} 2017}, 2017.

\end{thebibliography}

\newpage

\appendix 
\section{Relations to Algorithmic Stability of SGD}
\label{sec:stability}

In this section, we formally introduce notions of algorithmic stability and relate them to results presented in the paper.
Let $Z$ denote a set of datapoints and $\Z$ a distribution over $Z$.
For two training sets $S, S' \in Z^n$, we write $S \simeq S'$ if they differ in exactly one datapoint.
For a learning algorithm $\A \colon Z^* \to \R^d$, we define the \emph{uniform argument stability} (UAS) of $\A$ by
\begin{align}\label{eq:uas_A}
    \epstabA(n) \eqq \max_{S \simeq S', |S| = n} 
    \normb{\A(S) - \A(S')},
\end{align}
and the \emph{average argument stability} (AAS) of $\A$ by 
\begin{align}
    \epavgstabA(n) \eqq \max_{\Z} \cb{ \frac{1}{n} \sum_{i=1}^n \E_{S \sim \Z^n, z_i' \sim \Z} 
    \normb{\A(S) - \A(S^{(i)})}
    },
\end{align}
where $S^{(i)}$ is formed by taking $S$ and replacing $z_i$ with $z_i'$. 

It is well known (e.g., \cite{bousquet2002stability, shalev2010learnability}) that for any distribution $\Z$ and algorithm $\A$, the following relations holds between the generalization gap, AAS, and UAS;
\begin{align}
    \av{ \E_{S \sim \Z^n}\sb[b]{F(\A(S)) - \hF(\A(S))} }
    \leq
    \epavgstabA(n)
    \leq 
    \epstabA(n)
    \label{eq:stab_gen}
    .
\end{align}
In \cite{bassily2020stability} it was established the UAS of both with and without-replacement SGD is $\Omega(\eta\sqrt n)$ for $n$ steps of size $\eta$ with $n$ training examples. However,
considering \cref{eq:stab_gen}, it remained unclear whether the AAS and generalization gap of these algorithms exhibit rates of similar order, in which case the UAS accurately captures the rate of the generalization gap. 
Interestingly, the answer to this question depends on whether sampling is done with or without replacement, as we discuss next.

\paragraph{Stability of without-replacement SGD.}
As an immediate corollary of our \cref{thm:lb_main}, we have that the AAS of without-replacement SGD is also $\Omega(\eta\sqrt n)$. This follows from \cref{eq:stab_gen} and since the theorem establishes the generalization gap to be $\Omega(\eta \sqrt n)$.
Similarly, the lower bound given by \cref{thm:sc_lb_main} demonstrates that the AAS of SGD in the strongly convex case is $\Omega(1/\lambda \sqrt n)$. Combined with the naive upper bound argument for uniform stability of $O(1/\lambda \sqrt n)$ (which follows by convergence of SGD iterates to the minimizer in parameter space), we get a tight characterization of stability for strongly convex losses for the standard step size schedule.

% Furthermore, a short argument combining \cref{lem:sc_iterates} and \cref{lem:sgd_strong_stab} provides a tight characterization for an arbitrary step size schedule for the case that $\lambda \geq 1/\sqrt n$.

\paragraph{Stability of with-replacement (one-pass) SGD.}
In \cref{sec:wrsgd}, specifically in \cref{cor:wrsgd_gengap}, we establish a generalization gap of $O(1/\sqrt n)$ for with-replacement SGD with a particular averaging scheme and a properly tuned step size.
However, as it turns out, the average argument stability of this version of SGD is nonetheless of order $\Omega(\eta \sqrt n)$ as we demonstrate in \cref{thm:wrsgd_stab_lb} below.
This shows that this version of with-replacement SGD is an algorithm that is not stable in any sense but the most general one (namely, the one being equivalent to the generalization gap).

\begin{theorem}\label{thm:wrsgd_stab_lb}
    Let $n \in \N$, $n \geq 200$, $d \geq 2^{3 n}$, $W=\B_0^{d}(1)$. Further, let $\{\beta_t\}_{t=1}^n$ be an iterate averaging scheme that does not decay too quickly;
    $\sum_{s=t+1}^n \beta_n \geq C((n-t)/n)^2$ for some constant $C > 0$.
    Then there exists a distribution over the instance set $Z = \{0, 1\}^d$
    and a $2$-Lipschitz, convex loss function $f\colon W \times Z \to \R$ such that for all $k\in [n]$,
    \begin{align*}
        \E_{S \sim \Z^n, z_k' \sim \Z} 
        \normb{\hat w - \hat w^{(k)}}
        \geq \Omega(\eta\sqrt n),
    \end{align*}
    where $\hat w, \hat w^{(k)}$ denote the $\{\beta_n\}$-averaged iterates $\{w_t\}, \{w_t'\}$ of $n$ with-replacement SGD iterations (initialized at $w_1' = w_1 = 0$) with step size $\eta > 0$ over the training sets $S$ and $S^{(k)}$ respectively;
    \begin{align*}
        \hat w \eqq \sum_{t=1}^n\beta_n w_t;
        \quad 
        \hat w^{(k)} \eqq \sum_{t=1}^n\beta_n w_t'.
    \end{align*}
\end{theorem}

Note that the averaging scheme employed in \cref{thm:wrsgd_population_ub} decays sufficiently slow so as to satisfy requirements of \cref{thm:wrsgd_stab_lb}, hence the stability lower bound follows.

\begin{proof}
    Let $\Z$ be defined by $z(i) \sim \mathrm{Ber}(1/2)$, and set
    \begin{align*}
        f(w; z) := - \epsilon \sum_{i=1}^d \alpha_z(i) w(i) + \max_{i \in [d]} \{ w(i)\}.
    \end{align*}
    We will take $\epsilon \eqq \beta_n/(16n^3 d)$, and define
    \begin{align*}
        \alpha_z(i) = \begin{cases}
            -n \quad & z(i) = 1,
            \\
            1 \quad & z(i) = 0.
        \end{cases}
    \end{align*}
    In addition, let
    \begin{align*}
        I(w) \eqq \argmin_{i\in[d]} \left\{ i \mid w(i) = \max_j \{w(j)\} \right\},
    \end{align*}
    set
    \begin{align}
        i_t \eqq I(w_t),
        \quad
        i_t' \eqq I(w_t'),
    \end{align}
    and define 
    \begin{align}\label{eq:wrsgd_stab_g}
        g(w; z) \eqq -\epsilon \alpha_z + e_{I(w)},
    \end{align}
    where $e_i$ denotes the $i$'th standard basis vector. It follows that $g(w; z) \in \partial f(w; z)$ and $\norm{g(w; z)} \leq n d \epsilon + 1 \leq 2$ for all $z\in Z, w \in W$.
    Proceeding, we denote 
    \begin{align*}
        S = \{ z_1, \ldots, z_n \},
        \quad
        S' = \{ z_1', \ldots, z_n' \},
    \end{align*}
    and note that $z_l = z_l'$ for all $l \neq k$.
    Further, we denote the training examples sampled by SGD by
    \begin{align*}
        \hat z_t \eqq z_{k_t} \in S,
        \quad
        \hat z_t' \eqq z_{k_t}' \in S^{(k)},
    \end{align*}
    where $\{k_t\} \sim \Unif[n]$ are uniformly random and independent training indices. 
    For the remainder of the proof, we condition on the event 
    \begin{align}
        \mathcal E
         = \{ (z_1, \ldots, z_n, z_k')
         \mid 
         Z(S) \geq n,
         \;\; \text{and} \;\;
         Z(S') \geq n
        \},
        \label{eq:stab_good_event}
    \end{align}
    where $Z(S) \eqq |\{i \mid \forall r\in[n],\; z_r(i) = 0\}|$ and similarly $Z(S') \eqq |\{i \mid \forall r\in[n],\; z_r'(i) = 0\}|$.
    Owed to our assumption that $d \geq 2^{3n}$, a standard concentration argument shows this event occurs with probability $\geq 1/2$.
    
    We will now proceed to track how the SGD iterates evolve. Observe that for all $t\in [n]$, we have by direct computations of the gradient steps with \cref{eq:wrsgd_stab_g};
    \begin{align*}
        i\notin \{i_1, \ldots, i_t\}
        &\implies 
        w_{t+1}(i) = \eta \epsilon \sum_{s=1}^t \alpha_{\hat z_s}(i)
        \\
        &\implies
        \begin{cases}
            w_{t+1}(i) = \eta \epsilon t \quad &\forall s \leq t,\; \hat z_s(i) = 0,
            \\
            w_{t+1}(i) \in [-\eta\epsilon n^2, 0] \quad &\exists s \leq t,\; \hat z_s(i) = 1.
        \end{cases}
    \end{align*}
    In addition, from similar computations; 
    \begin{align*}
        i \in \{i_1, \ldots, i_t\}
        &\implies 
        w_{t+1}(i) \leq -\eta + \eta \epsilon n
        .
    \end{align*}
    Summarizing, and applying identical calculations for $w_{t+1}'$, we have:
    \begin{alignat}{3}
        t < s
        &\implies 
            w_s(i_t) &\leq -\eta + \eta \epsilon n,
        \quad
        i \notin \{i_1, \ldots, i_n\}
        &\implies 
        \forall s, \; 
        w_s(i) &\in [-\eta\epsilon n^2, \eta \epsilon n]
        \nonumber \\
        t < s 
        &\implies 
            w'_s(i'_t) &\leq -\eta + \eta \epsilon n,
        \quad
        i' \notin \{i'_1, \ldots, i'_n\}
        &\implies 
        \forall s,\; w'_s(i') &\in [-\eta\epsilon n^2, \eta \epsilon n].
        \label{eq:wrsgd_iters}
    \end{alignat}
    By \cref{eq:wrsgd_iters} above, for all $t\in[n]$ we have;
    \begin{align}
        \hat w (i_t)
        = \sum_{s=1}^n \beta_s w_s(i_t)
        &\leq 
            \sum_{s=1}^t \beta_s \eta s \epsilon
            - \sum_{s=t+1}^{n} \beta_s (\eta - \eta \epsilon n)
        \nonumber \\
        &\leq
        - \eta \sum_{s=t+1}^{n} \beta_s  
        + \eta n \epsilon \sum_{s=1}^n \beta_s 
        \leq -\frac{3}{4}\eta \sum_{s=t+1}^{n} \beta_s,
        \label{eq:wrsgd_wit}
    \end{align}
    where the last inequality follows from our choice of $\epsilon$.
    In addition, if $i_t \notin \{i_1', \ldots i_n' \}$, again by \cref{eq:wrsgd_iters} and our choice of $\epsilon$ it follows that;
    \begin{align}
        \hat w^{(k)}(i_t)
        = \sum_{s=1}^n \beta_s w'_s(i_t)
        \geq \sum_{s=1}^n \beta_s \eta n^2 \epsilon
        \geq - \frac{\eta}{4} \beta_n.
        \label{eq:wrsgd_wit_prime}
    \end{align}
    Now, set $t_0 \eqq \min\{t \mid k_t = k\}$ to be the first time that index $k$ (in which training examples differ) is chosen, and let $t > t_0$.
    Note that $z'_k(i_t) = 1$ implies $i_t \notin \{ i'_1, \ldots, i_n'\}$; to see this, observe that for $\tau\leq t_0$, $i'_\tau = i_\tau$, while $\tau > t_0$ implies $i'_\tau \neq i_t$, since 
    the event \cref{eq:stab_good_event} we condition on ensures
    \begin{align*}
        i_t = \min \{ i \in [d] \mid \forall s < t,\; \hat z_s(i) = 0  \},
        \\
        i_t'= \min \{ i \in [d] \mid \forall s < t,\; \hat z'_s(i) = 0  \}.
    \end{align*}
    (From the above it also must hold that $i_t \neq i_\tau, i'_t \neq i'_\tau$ for all $t \neq \tau$.)
    Thus, putting together \cref{eq:wrsgd_wit}, \cref{eq:wrsgd_wit_prime} and the fact that $z'_k(i_t) = 1$ implies $i_t \notin \{ i'_1, \ldots, i_n'\}$, we obtain for all $t_0 < t$;
    \begin{align*}
        |\hat w(i_t) - \hat w^{(k)}(i_t)|
        \geq 
            \Ind{z'_k(i_t) = 1}
            \frac{\eta}{2} \sum_{s=t+1}^n \beta_s
        \geq 
            \Ind{z'_k(i_t) = 1}
            \frac{\eta C (n-t)^2}{2 n^2}
        ,
    \end{align*}
    where the second inequality follows from our assumption on $\{\beta_t\}$.
    Thus, for $t_0 < t\leq 3n/4$, we get that 
    \begin{align*}
        |\hat w(i_t) - \hat w^{(k)}(i_t)|
        &\geq \Ind{z'_k(i_t) = 1}
        \frac{C}{4} \eta,
    \end{align*}
    and taking expectations we obtain;
    \begin{align*}
        \E \sb{ \norm{\hat w - \hat w^{(k)}} \mid t_0 }
        &\geq \E \sb{ \sqrt {
            \sum_{t=1}^n (\hat w(i_t) - \hat w'(i_t))^2
        } \mid t_0}
        \\
        &\geq \frac{C \eta}{4}  \E \sb{ \sqrt { \sum_{t=t_0+1}^{3n/4}
            \Ind{z'_k(i_t) = 1}
        }\mid t_0}.
    \end{align*}
    Now, observe that $z_k'$ is independent of $i_t$ for all $t$, hence the expectation above is of the form
    \begin{align*}
        \E \sqrt{\sum_{l = 1}^{m} Y_l},
    \end{align*}
    where $m \eqq 3n/4-t_0$ and $Y_l \sim \Ber(1/2)$ are independent.
    Thus,
    \begin{align*}
        \Pr \b{ m/2 - \sum_{l=1}^m Y_t > m/4 }
        \leq e^{-m/16} \leq 1/2,
    \end{align*}
    for $m > 100$, and then 
    \begin{align*}
        \E \sb{ \norm{\hat w - \hat w^{(k)}} \mid t_0 }
        \geq \frac{C \eta}{4}\frac{1}{2} \sqrt{ 3n/4 - t_0}
        =
        \frac{C \eta}{8} \sqrt{ 3n/4 - t_0}
        .
    \end{align*}
    To conclude, we use the fact that $t_0$ follows a geometric distribution with parameter $1/n$, therefore 
    \begin{align*}
        \Pr(t_0 \leq n/2) = \frac{1}{n}\sum_{t=1}^{n/2} (1 - 1/n)^t
        = 1 - (1 - 1/n)^{n/2}
        \geq 1 - e^{-1/2}
        \geq 1/3.
    \end{align*}
    This implies,
    \begin{align*}
        \E \norm{\hat w - \hat w^{(k)}}
        \geq \frac{C \eta}{32} \sqrt{ 3n/4 - n/2}
        = \frac{C \eta}{64} \sqrt{n},
    \end{align*}
    and completes the proof.
\end{proof}

\section[Main LB]{Proof of~\cref{thm:lb_main}}
\label{sec:proof:lb_main}
Our first proof below applies for step sizes $\eta \leq 1/\sqrt n$. The extension for larger step sizes is rather technical and requires care of the projection step --- we provide it in \cref{sec:proof:main_lb_proj}. 
The statement of \cref{thm:lb_main} is repeated below for the case of the small step size regime.
\begin{theorem}[Small step size case of \cref{thm:lb_main}]
    Let $n \in \N$, $n \geq 4$, $d \geq 2^{4 n\log n}$, and $W=\B_0^{2d}(1)$.
    Then there exists a distribution over instance set $Z$
    and a $4$-Lipschitz convex loss function $f\colon W \times Z \to \R$ such that 
    \begin{enumerate}[label=(\roman*)]
        \item the optimization error is large;
        \begin{aligni*}
            \E_{S \sim \Z^n}
            \sb{ \hF(\hw) - \hF(\wERM)} 
            = \Omega \b{ \eta \sqrt n },
        \end{aligni*}
        \item the generalization gap is large;
        \begin{aligni*}
            \E_{S \sim \Z^n}
            \sb{ \hF(\hw) - F(\hw)} 
            = \Omega \b{ \eta \sqrt n },
        \end{aligni*}
    \end{enumerate}
    where $\hat w$ is any suffix average of SGD with step size $\eta \leq 1/\sqrt n$.
\end{theorem}
\begin{proof}
    Our construction is parameterized by $\epsilon, \delta > 0$, which will be chosen later.
    We will work with the datapoints set $Z=\{0, 1\}^{2d}$ and define the distribution $\Z = \Z(\delta)$ over $Z$ by
    \begin{align*}
        \forall i \leq d; \quad 
        z(i) &= \begin{cases}
            1 \quad \text{w.p. } \delta ;
            \\
            0 \quad \text{w.p. } (1-\delta),
        \end{cases}
        \\
        \forall i > d ; \quad 
        z(i) &= z(i-d).
    \end{align*}
    Our loss function is a combination of two components; the ``push'' function $\phi$ is in charge of driving the SGD iterate towards areas in the $L_2$ ball where the ``penalty'' function $\nu$ inflicts a norm-like loss.
    \begin{align}
        \phi(w; z) &\eqq - \epsilon \sum_{i=d+1}^{2d} z(i)w(i) 
            + \max_{1 \leq i\leq d} \cb{w(i) + w(i + d)},
        \label{eq:main_phi}
        \\
        \nu_z(w) &\eqq 
        \sqrt{ \sum_{i=1}^d z(i)w(i)^2 },
        \label{eq:main_nu}
        \\
        f(w; z) &\eqq
        \phi(w; z)
        + \nu_z(w).
        \nonumber
     \end{align}
    The lower bound arguments all go through with any sub-gradient oracle $g(w; z) \in \partial_w f(w; z)$. For clarity of exposition, we make use of the gradient oracle $g$ that returns the minimal coordinate sub-gradient for the max component in $\phi$;
    \begin{align}
        g_\phi(w; z)(i) &\eqq 
        \begin{cases}
            \Ind{i = I(w)} \quad & i \leq d
            \\
            - \epsilon z(i) + \Ind{i = I(w)+d}  \quad & i \geq d
        \end{cases},
        \label{eq:main_grad}
        \\
        \text{where }
        I(w) &\eqq 
        \min\Big\{ i \in [d] \mid 
            i \in \argmax_{1 \leq j \leq d} 
            \{ w(j) + w(j+ d) \}
        \Big\}.
        \label{eq:main_grad_I}
    \end{align}
    We additionally denote the index picked by $g$ on round $t\in[n]$ by
    \begin{align}
        i_t \eqq I(w_t).
        \label{eq:main_grad_i_t}
    \end{align}
    We then set
    \begin{aligni*}
        g(w; z) 
        \eqq g_\phi(w; z) + \nabla \nu_z(w).
    \end{aligni*}
    It now follows that for all $w, z\in \R^{2d}$, $g_\phi(w; z) \in \partial_w \phi(w; z)$, thus $g(w; z) \in \partial_w f(w; z) $.
    Choosing $\epsilon=1/d$, we get that $f$ is $4$-Lipschitz; \begin{align*}
        \norm{g(w; z)}
        \leq 
        \epsilon \sqrt d
        + 2
        + \frac{1}{2\norm{w}}\norm{w}
        \leq 3 + 1 / \sqrt d
        \leq 4.
    \end{align*}
    With the above construction in place, we first claim that with sufficiently large probability, the training set will contain the desired collection of ``bad'' coordinates which will be picked up by our gradient oracle.
    Indeed, with the dimension $d$ large enough, a proper choice of $\delta$ ensures that for every $t\in[n]$, there will be a certain coordinate with a prefix of $t-1$ ones followed by a zero only suffix. 
    \begin{lemma}\label{lem:main_Z}
        For $\delta=1/4n^2$, with probability $\geq 1/2$ over the random draw of $S = \cb{z_1, \ldots, z_n} \sim \Z^n$, it holds that for all $t\in[n]$:
        \begin{enumerate}
            \item There exist prefix of ones coordinates 
            $E^{(1,t)} \eqq \cb{ j \in[d] \mid s < t \implies z_s(j) = 1} \neq \emptyset$, and
            \item the minimal such coordinate $\io_t \eqq \min \cb{j \in E^{(1,t)}}$  also has a zero suffix; $s \geq t \implies z_s(\io_t) = 0$.
        \end{enumerate}
    \end{lemma}
    From this point onward fix $\delta \eqq 1/4n^2$.
    By definition of our gradient oracle, a relatively straightforward argument given in our next lemma establishes SGD will take gradient steps precisely on those bad coordinates of \cref{lem:main_Z}.
    Notably, we have designed the construction so that these steps are made only after the samples penalizing those coordinates have been processed. This eliminates the possibility for SGD to correct these coordinates in future steps. 
    \begin{lemma}\label{lem:lb_iterates}
        We have with probability $\geq 1/2$ that for all $t \in [n]$, $i_t = \io_t$ (see Eq. \ref{eq:main_grad_i_t}), 
        and for all $\tau \in [n], \tau > t$, 
        $w_{\tau}(i_t) = -\eta$.
    \end{lemma}
    To complete the proof, we assume the event from the lemma occurs. Since it occurs with constant probability, a lower bound derived conditioned on it implies a lower bound in expectation.
    First we argue the population loss of all iterates is upper bounded as 
    \begin{align*}
        F(w_t) 
    &= - \epsilon \sum_{i=d+1}^{2d} \delta w_t(i) 
        + \epsilon n \eta 
        + \mathbb E \sb{ \sqrt{\sum_{i=1}^d z(i)w_t(i)^2 }}
    \\
    &\leq -\epsilon \delta \sum_{s=1}^{t-1} w_t(i_s) 
        + \epsilon n \eta 
        + \sqrt{\sum_{i=1}^d \mathbb E [z(i)]w_t(i)^2 }
    \\
    &= \epsilon \delta (t-1)\eta
        + \epsilon n \eta 
        + \sqrt{\delta \sum_{i=1}^d  w_t(i)^2 }
    \\
    &\leq 2 \epsilon n \eta + 
        \sqrt{\delta \sum_{s=1}^t  w_t(i_s)^2 }
    \\
    &\leq 2 \epsilon n \eta + 
        \sqrt{\delta \eta^2n}
    = \frac{2 n \eta }{d} + \frac{\eta \sqrt n}{2n}
    \leq \frac{\eta}{\sqrt n}.
    \end{align*}
    By convexity of the population loss, the above implies that any suffix average satisfies $F(\hat w) \leq \eta/\sqrt n$. 
    In addition, note that $\hF(\wERM) \leq \hF(0) \leq 0$, hence the $\Omega(\eta \sqrt n)$ bound we will now establish on the empirical risk of SGD implies our claimed optimization and generalization lower bounds.
    Indeed, let $\tau\in [n]$ and denote $\wbar w_{\tau} \eqq \frac{1}{n-\tau+2} \sum_{t=\tau}^{n+1} w_t$. 
    Observe that for $1 \leq t \leq n/2$, by \cref{lem:lb_iterates} at least half of the iterates have the $-\eta$ value in coordinate $i_t$;
    \begin{align*}
        \wbar w_{\tau} (i_t) 
        = \frac{1}{n-\tau+2}\sum_{s=\tau}^{n+1} w_s(i_t)
        &\leq \frac{1}{n-\tau+2}\sum_{s=\max\cb{\tau, n/2}}^{n+1} w_s(i_t)
        \\
        &\leq \frac{n - \max\cb{\tau, n/2} + 2}{n-\tau+2} \b{-\eta}
        \leq - \frac{\eta}{2}.
    \end{align*}
    (We ignore the fact that the last iterate, formally speaking, may have a slightly greater value due to the projection on the last step.)
    Now, for any $w\in W$,
    \begin{align*}
        \hF(w) 
        = \frac{1}{n} \sum_{s=1}^n f(w; z_s) 
        &= \frac{1}{n} \sum_{s=1}^n \phi(w; z_s)
        + \frac{1}{n} \sum_{s=1}^n \sqrt{\sum_{t=s+1}^n w(i_t)^2} 
        \\
        &\geq \frac{1}{n} \sum_{s=1}^n \phi(w; z_s) + \frac{1}{5 \sqrt n} \sum_{t=n/4}^n \av{w(i_t)},
    \end{align*}
    where the second inequality follows from \cref{lem:loss_lb_util}.
    Noting that $\phi(\wbar w_{\tau}; z_s) \geq -\epsilon^2 d n \eta \geq -\eta$ and combining the last two displays we get that
    \begin{align*}
        \hF(\wbar w_{\tau:n}) 
        \geq \frac{1}{5 \sqrt n} \sum_{t=n/4}^{n/2} \frac{\eta}{2} - \eta
        = \frac{\eta (\sqrt n - 1)}{40},
    \end{align*}
    which completes the proof.
\end{proof}

\begin{proof}[of \cref{lem:main_Z}]
    Fix $t\in[n]$, denote
    \begin{align*}
        E^{(1,t)} &\eqq \big\{i\in [d] 
            \mid z_s(i) = 1 \; \forall s < t
        \big\},
        \\
        E^{(t, 0)} &\eqq \big\{i\in [d] 
            \mid z_s(i) = 0 \; \forall s \geq t
        \big\},
    \end{align*}
    and let $\io_t \in E^{(1,t)}$ be the minimal element if it is not empty.
    Note that
    \begin{align*}
        \Pr(E^{(1,t)} = \emptyset) 
        = \Pr\b{ \forall i\in[d],\; \exists s < t,\; z_s(i) = 0} 
        = \b{ 1 - \delta^{t-1} }^d
        \leq \b{ 1 - \delta^n }^d.
    \end{align*}
    In addition, since the contents of $E^{(1,t)}$ are independent of $z_t, \ldots, z_n$, we have that for any $i \in E^{(1,t)}$,
    \begin{align*}
        \Pr(i \in E^{(t, 0)})
        = (1-\delta)^{n-t+1} \geq (1-\delta)^n.
    \end{align*}
    Therefore,
    \begin{align*}
        \Pr\b{E^{(1,t)} = \emptyset \;\text{ OR }\; 
            \b{E^{(1,t)} \neq \emptyset \;\text{ but }\; \io_t\notin E^{(t, 0)}}}
        \leq \b{ 1 - \delta^n }^d  + 1-(1-\delta)^n
        .
    \end{align*}
    Now, by the union bound over all values of $t\in[n]$ we obtain
    \begin{align}
        \Pr \b{ \forall t\in [n],\; 
            E^{(1,t)} \neq \emptyset \text{ AND } \io_t\in E^{(t, 0)}
         } 
         \geq 1 - n \b{ \b{ 1 - \delta^n }^d  + 1-(1-\delta)^n }.
        \label{eq:lem_Z_1}
    \end{align}
    Now, since $\delta=1/4n^2$ we have
    \begin{align*}
        \b{ 1 - \delta^n }^d 
        = \b{ 1 - \frac{d\delta^n}{d} }^d
        \leq e^{-d \delta^n}
        \leq \frac{1}{4 n},
    \end{align*}
    where the last inequality follows for $d \geq \delta^{-n}\log(4 n) = 4^n n^{2n}\log(4 n)$
    (recall that by the assumption in the theorem statement $d\geq 2^{4n \log n} \geq 4^{n} n^{2n} 4 n$).
    In addition;
    \begin{align*}
        (1-\delta)^n
        =
        \left( 1 - \frac{1}{4 n^2} \right)^n
        \geq 1 - \frac{1}{4 n}
        \implies 
        1 - (1 - \delta)^n
        \leq \frac{1}{4 n}.
    \end{align*}
    Back to \cref{eq:lem_Z_1}, applying the inequalities from the last two displays we obtain the desired event occurs with probability
    \begin{align}
        \geq 1 - n \b{ \b{ 1 - \delta^n }^d  + 1-(1-\delta)^n }
        \geq 1 - n \b{ \frac{1}{4 n}  + \frac{1}{4 n} }
        = \frac{1}{2},
    \end{align}
    and the result follows.
\end{proof}

\begin{proof}[of \cref{lem:lb_iterates}]
    Following a direct computation, we get that
    \begin{align}
        g(w; z)(i) &= 
        \begin{cases}
            \Ind{i = I(w)} + \frac{z(i) w(i)}{\nu_z(w)}
            \quad & i \leq d,
            \\
            \Ind{i = I(w)+d} - \epsilon z(i) \quad & i > d.
        \end{cases}
        \label{eq:main_lb_grads}
    \end{align}
    From the above we see that the value of $w(i+d)$ for every coordinate $i + d\in \cb{d+1, \ldots, 2d}$ gains $\eta\epsilon$ when $z(i) = 1$, while the value of coordinate $i$ only decreases. Thus $I(w_t)$ will be a coordinate with an all ones prefix if one exists.
    Formally, let $t\in[n]$, and observe that our gradient oracle will return the minimal coordinate $i_t\in[d]$ with the maximum value of $w_t(i_t) + w_t(i_t + d)$. 
    Assuming the event from \cref{lem:main_Z} occurs, note that the coordinate $\io_t \in[d]$ with 
    $z_1(\io_t) = \ldots = z_{t-1}(\io_t) = 1$ exists.
    Now, observe that any coordinate $j\in[d]$ is bound to satisfy
    \begin{align*}
        w_t(j) + w_t(j + d) \leq w_t(\io_t) + w_t(\io_t + d).
    \end{align*}
    To see this, note that $w_t(\io_t) = 0$, because $\io_t \neq i_s$ for all $s < t$ (formally this follows by induction). In addition, by \cref{eq:main_lb_grads};
    \begin{align*}
        \forall s < t, \; z_s(\io_t) = 1
        \implies 
        \forall s < t, \;
        -\eta g(w; z_s)(\io_t+d) = \epsilon \eta
        &\implies
        w_t(\io_t+d) = (t-1)\epsilon \eta
        .
    \end{align*}
    On the other hand, for any $j'\in[d]$ we have 
    $w_t(j')\leq 0$, and 
    $w_t(j'+d) \leq (t-1)\epsilon \eta$.
    Concluding, it follows the gradient oracle will pick $i_t = I(w_t) = \io_t$, therefore $w_{t+1}'(i_t) = -\eta$ for $w_{t+1}' \eqq w_t - \eta g_t$.
    To see that $w_{t+1} = \Pi_W(w_{t+1}') = w_{t+1}'$, note that
    by assumption $\eta\leq 1/\sqrt n$, hence
    \begin{align*}
        \norm{w_{t+1}'}^2 
        = \sum_{i=1}^{2d} w_{t+1}'(i)^2 
        \leq \sum_{t=1}^{t} w_{t+1}'(i_t)^2 
        + d(n\epsilon \eta )^2
        = \eta^2 (t + n/d)
        \leq (t + 1)/n,
    \end{align*}
    and $w_{t+1}'$ remains inside $W$ for all $t < n$.
    Finally, since the desired event occurs with probability $1/2$ by \cref{lem:main_Z}, we are done.
\end{proof}

\subsection{Lower bound for large step sizes}
\label{sec:proof:main_lb_proj}
When the step size is large, the projections actually alleviate the problematic nature of our construction, to the point where they can be exploited to obtain any convergence rate with the full iterate averaging.
Notably though, concatenating our construction with a standard lower bound (e.g., \cref{lem:std_convex_lb}) the best convergence rate possible is $n^{-1/4}$ with $\eta=n^{-1/4}$ which is the same as what would be achieved by the somewhat more reasonable choice of $\eta=n^{-3/4}$ that does not rely on the projections.
\begin{theorem}[Large step size case of \cref{thm:lb_main}]
    Let $n \in \N$, $n \geq 4$, $d \geq 2^{4 n\log n}$, and $W=\B_0^{2d}(1)$.
    Then there exists a distribution over instance set $Z$
    and a $4$-Lipschitz convex loss function $f\colon W \times Z \to \R$ such that 
    \begin{enumerate}[label=(\roman*)]
        \item the optimization error is large;
        \begin{aligni*}
            \E_{S \sim \Z^n}
            \sb{ \hF(\hw) - \hF(\wERM)} 
            = \Omega \b{ \frac{1}{\eta \sqrt n } },
        \end{aligni*}
        \item the generalization gap is large;
        \begin{aligni*}
            \E_{S \sim \Z^n}
            \sb{ \hF(\hw) - F(\hw)} 
            = \Omega \b{ \frac{1}{\eta \sqrt n } },
        \end{aligni*}
    \end{enumerate}
    % \us{TODO}Let $n \in \N$, $n \geq 4$, $d \geq 2^{8 n\log n}$, and $W=\B_0^{2d}(1)$.
    % Then there exists a distribution over instance set $Z$
    % and a $4$-Lipschitz convex loss function $f\colon W \times Z \to \R$ such that 
    % \begin{align*}
    %     \E_{S \sim \Z^n}
    %     \sb{ \hat F(\hat w) - \hat F(\wERM)} 
    %     &= \Omega \b{ \frac{1}{\eta \sqrt n }}
    % \end{align*}
    where $\hat w$ is any suffix average of SGD with step size $\eta > 1/\sqrt n$.
\end{theorem}
\begin{proof}
    The analysis parts ways from the small step size case after \cref{lem:main_Z}.
    Instead of \cref{lem:lb_iterates}, we make the claim below.
    \begin{lemma}\label{lem:lb_iterates_proj}
        For all $\tau\in[n]$, it holds that
      \begin{aligni*}
          t < \tau \implies w_{\tau}(i_t) \leq -\eta (1+\eta^2)^{t-\tau},
      \end{aligni*}
      where $i_t \eqq I(w_t)$ and $\tau \leq n+1$.
      In addition, for any suffix average $\hat w$, it holds that
      \begin{align*}
          \sum_{t=n/4}^n \av{\hat w(i_t)} \geq \frac{1}{20 \eta}.
      \end{align*}
    \end{lemma}
    The important consequence of the above lemma is that whichever suffix average we take, we will end up with an $\Omega(1/\eta)$ mass in the total bad coordinate summation.
    We now show this translates to an empirical risk lower bound as claimed.
    Ignoring the negligible contribution of $\epsilon$, by \cref{lem:loss_lb_util} we have;
    \begin{align*}
        \hF(w) 
        = \frac{1}{n} \sum_{s=1}^n f(w; z_s) 
        \geq \frac{1}{n} \sum_{s=1}^n \sqrt{\sum_{t=s+1}^n w(i_t)^2} 
        \geq \frac{1}{5 \sqrt n} \sum_{t=n/4}^n \av{w(i_t)}
        \geq \frac{1}{100 \eta \sqrt n},
    \end{align*}
    where the last inequality follows from \cref{lem:lb_iterates_proj}.
    This completes the proof.
\end{proof}

\begin{proof}[of \cref{lem:lb_iterates_proj}]
    For $t\in[n]$, denote 
    $w_t' \eqq w_t - \eta g(w_t, z_t)$ so that now $w_{t+1} \gets \Pi_W\b {w_t'}$.
    Informally, we have $\norm{w_t'}^2 \leq 1 + \eta^2$ for all $t$, when we ignore the negligible $\epsilon$ component. Formally,
    let 
    \begin{align*}
        \zeta_t(i) = 
        \begin{cases}
            0
            \quad & i \leq d,
            \\
            - \epsilon z_t(i) \quad & i > d,
        \end{cases}
    \end{align*}
    so that $g(w_t; z_t) = \zeta_t + e_{i_t} + e_{(i_t + d)}$, 
    and observe
    \begin{align*}
        \norm{w_{t+1}'}^2
        &= \norm{w_t - \eta \zeta_t - \eta e_{i_t} - \eta e_{(t_t + d)}}^2
        \\
        &\leq 1 + d n\eta^2\epsilon^2 + 2 \eta^2
        \\
        &\leq 1 + (n/d)\eta^2 + 2 \eta^2
        \\
        &\leq 1 + 3 \eta^2
        .
    \end{align*}
    Now, set $\gamma \eqq 3\eta^2$, let $t < \tau \in[n]$, and observe;
        \begin{align*}
            w_{\tau}(i_t) = \Pi_W\b {w_{\tau}'} (i_t)
            = \frac{w_{\tau}'(i_t)}{\norm{w_{\tau}'}}
            &\leq w_{\tau}'(i_t) (1+\gamma)^{-1}
        \end{align*}
    where the inequality follows from the norm bound and since $w_{l}'(i_t) \leq 0$ for all $l\in[d]$.
    In addition, for $\tau - 1 > t$, we have $w_{\tau}'(i_t) = w_{\tau-1}(i_t)$, hence
    \begin{align*}
        w_{\tau}'(i_t) (1+\gamma)^{-1} 
        = w_{\tau-1}(i_t) (1+\gamma)^{-1} 
        \leq \cdots \leq w_{t+1}'(i_t) (1+\gamma)^{t - \tau}
        = -\eta (1+\gamma)^{t - \tau},
    \end{align*}
    therefore,
        \begin{align}
            t < \tau \implies w_{\tau}(i_t) \leq -\eta (1+\gamma)^{t-\tau},
            \label{eq:lb_proj_1}
        \end{align}
    which proves the first part.
    For the second part, we begin by computing the values in each individual coordinate.
    \paragraph{The individual coordiantes $\boldsymbol {\hat w(i_t)}$.}
    let $\hat w$ be the average of the last $k$ iterates
    $w_{n-k+2}, \ldots, w_{n+1}$, and set $\tau_0 \eqq n - k$.
    Fix $t\in[n]$, set $l \eqq \max(\tau_0 + 1, t+1)$, and observe
    \begin{align*}
        \av{\hat w (i_t)} 
            \geq \frac{1}{k} 
                \sum_{\tau=l}^{n+1} 
                \av{ w_{\tau} (i_t) }
            \geq \frac{\eta}{k} 
                \sum_{\tau=l}^{n+1} 
                (1+\eta^2)^{t - \tau },
    \end{align*}
    where the first inequality follows since all values are negative and the second from \cref{eq:lb_proj_1}.
    We have
    \begin{align*}
        \sum_{\tau=l}^{n+1} (1+\gamma)^{t - \tau }
        &= \sum_{k=l-t}^{n+1-t} (1+\gamma)^{- k }
        \\
        &=(1+\gamma)^{- (l - t) } \b{ 1 -  (1+\gamma)^{- (n+2-l) } }
        \frac{
            1
        }{
            1 - (1+\gamma)^{-1}
        }
        \\
        &=(1+\gamma)^{t - l } \b{ 1 -  (1+\gamma)^{ l - n - 2} }
        \frac{
            (1+\gamma)
        }{
            (1+\gamma) - 1
        }
        \\
        &=(1+\gamma)^{t - l +1} \b{ 1 -  (1+\gamma)^{ l - n - 2} }
        \frac{
            1
        }{
            \gamma
        }
        \\
        & \eqqcolon (*)
        .
    \end{align*}
    Now, 
    \begin{align}
        t \leq \tau_0 
        &\implies
        (*) =
            (1+\gamma)^{t + k - n - 1} \b{ 1 -  (1+\gamma)^{ - k } }
                \frac{1}{\gamma}
        \nonumber \\
        &\implies
        \av{ \hat w(i_t) } 
            \geq \frac{\eta}{\gamma k} 
            (1+\gamma)^{t + k - n-1} \b{ 1 -  (1+\gamma)^{ - k } }
        \label{eq:lb_proj_t_small}
    \\ 
    \text{and }
        t > \tau_0 
        &\implies
        (*) = \b{ 1 -  (1+\gamma)^{ t - n - 1} }
            \frac{1}{\gamma}
        \nonumber \\
        &\implies
        \av{ \hat w(i_t) } \geq \frac{\eta}{\gamma k} 
            \b{ 1 -  (1+\gamma)^{ t - n - 1} }.
        \label{eq:lb_proj_t_large}
    \end{align}
    
    Before moving on to bound the sum of values in the coordinates, we record the following basic facts which will be used repeatedly.
    By Bernoulli's inequality and our assumption that $\eta > 1/\sqrt n$, we have $(1+ \gamma)^m = (1+3 \eta^2)^m \geq 1 + 3 \eta^2 m \geq 1 + 3 m/n$. Hence,
    \begin{align*}
        (1+\gamma)^{-m}
        = \frac{1}{(1+\gamma)^{m}}
        \leq \frac{1}{1+3 m/n}
        \\
        \implies 
        1 - (1+\gamma)^{-m} \geq \frac{3 m/n}{1+ 3m/n},
    \end{align*}
    and then,
    \begin{align}
        m\geq n/4 \implies 
        1 - (1+\gamma)^{-m} 
        &\geq \fraci{(3/4)}{4} \geq 1/6,
        \label{eq:lb_proj_exp_lb}
        \\ 
        \text{and }\;
        \sum_{j=0}^{m} (1+\gamma)^{-j}
        &\geq \frac{1 - (1+\gamma)^{-m}}{1 - (1+\gamma)^{-1}}
        \geq \frac{1/6}{1 - \frac{1}{1+\gamma}}
        = \frac{1}{6\gamma}.
        \label{eq:lb_proj_series_lb}
    \end{align}
    
    \paragraph{Bounding the sum of coordinate values.}
    We first consider the case that $\tau_0 < n/2$;
        \begin{align*}
            \sum_{t=n/4}^{n} \av{\hat w(i_t)}
            &\geq 
                \frac{\eta}{\gamma k} 
                \sum_{t=n/2}^{3n/4}  \b{ 1 -  (1+\eta^2)^{ t - n - 1} }
            \tag{by \cref{eq:lb_proj_t_large}}
            \\
            &\geq \frac{\eta}{\gamma k} 
                \sum_{t=n/2}^{3n/4} \frac{1}{6}
            \tag{by \cref{eq:lb_proj_exp_lb}}
            \\
            &= \frac{\eta}{6 \gamma k}(3n/4 - n/2)
            \\
            &= \frac{n \eta }{24 \gamma k}
            \\
            &\geq \frac{n}{100 \eta k}
            \geq \frac{1}{100 \eta }
            ,
        \end{align*}
    which proves the desired result (recall that $\gamma=3\eta^2$).
    Assume now $\tau_0 \geq n/2$, and observe;
        \begin{align}
            \sum_{t=\tau_0+1}^{n} \av{\hat w(i_t)}
            &\geq 
                \frac{\eta}{\gamma k} 
                \sum_{t=\tau_0+1}^{n}  \b{ 1 -  (1+\gamma)^{ t - n - 1} }
            \tag{by \cref{eq:lb_proj_t_large}}
            \nonumber \\
            &= \frac{\eta}{\gamma k} 
                \sum_{j=1}^{k-1}  \b{ 1 -  (1+\gamma)^{ -j} }
            \nonumber \\
            &= \frac{\eta}{\gamma k} \b{
                k - 1 - \sum_{j=1}^{k-1}  (1+\gamma)^{ -j} 
            }
            \nonumber \\
            &= \frac{\eta}{\gamma k} \b{
                k-1 - (1+\gamma)^{-1} \frac{
                    1 - (1+\gamma)^{1-k}
                    }{1 - (1+\gamma)^{-1}
                    }
            }
            \nonumber \\
            &= \frac{\eta}{\gamma k} \b{
                k-1 - \frac{
                    1 - (1+\gamma)^{1-k}
                    }{\gamma
                    }
            }
            \nonumber \\
            &= \frac{\eta}{\gamma^2 k} \b{
                \gamma(k-1) 
                - (1 - (1+\gamma)^{1-k})
            }
            \nonumber \\
            &= \frac{\eta}{\gamma^2 k} \b{
                \gamma (k-1)
                + (1+\gamma)^{1-k}
            }
            \label{eq:proj_lb_main_t_large}
            .
        \end{align}
    
    In addition,
        \begin{align}
            \sum_{t=n/4}^{\tau_0} \av{\hat w(i_t)}
            &\geq 
                \frac{\eta}{\gamma k} \b{ 1 -  (1+\gamma)^{ - k } }
                \sum_{t=n/4}^{\tau_0}  
                    (1+\gamma)^{t + k - n - 1} 
                \tag{by \cref{eq:lb_proj_t_small}}
            \nonumber \\
            &=
                \frac{\eta}{\gamma k} \b{ 1 -  (1+\gamma)^{ - k } }
                \sum_{t=n/4}^{\tau_0}  
                    (1+\gamma)^{t - \tau_0} 
                \tag{by \cref{eq:lb_proj_t_small}}
            \nonumber \\
            &=
                \frac{\eta}{\gamma k} \b{ 1 -  (1+\gamma)^{ - k } }
                \sum_{j=0}^{\tau_0-n/4}  
                    (1+\gamma)^{-j}
            \nonumber \\
            &\geq 
                \frac{\eta}{\gamma k} \b{ 1 -  (1+\gamma)^{ - k } }
                \sum_{j=0}^{n/4}  
                    (1+\gamma)^{-j}
            \nonumber \\
            &\geq 
                \frac{\eta}{6\gamma^2 k} \b{ 1 -  (1+\gamma)^{ - k } }
            \label{eq:proj_lb_main_t_small},
        \end{align}
    where in the last inequality we have applied \cref{eq:lb_proj_series_lb}.
    Now, combining \cref{eq:proj_lb_main_t_large} and \cref{eq:proj_lb_main_t_small}, we obtain
        \begin{align*}
            \sum_{t=n/4}^n \av{\hat w(i_t)}
            &\geq 
            \frac{\eta}{\gamma^2 k} \b{
                \gamma k
                - \gamma - 1
                + (1+\gamma)^{1-k}
            }
            + \frac{\eta}{6\gamma^2 k} \b{ 1 -  (1+\gamma)^{ - k } }
            \\
            &\geq \frac{\eta}{6\gamma^2 k} \b{
                \gamma k
                - \gamma - 1
                + (1+\gamma)^{1-k}
                +  1 
                -  (1+\gamma)^{ - k }
            }
            \\
            &= \frac{\eta}{6\gamma^2 k} \b{
                \gamma k
                - \gamma
                + (1+\gamma)^{1-k}
                -  (1+\gamma)^{ - k }
            }
            \\
            &\geq \frac{\eta }{6\gamma k }(k-1)
            \\
            &\geq \frac{\eta }{12 \gamma }
            \\
            &= \frac{1 }{36 \eta }
            ,
        \end{align*}
    which proves the desired result also in the second case, and completes the proof.
\end{proof}

\subsection{SCO with non convex components}
\label{sec:proof:lb_non_convex}

\begin{proof}[of \cref{thm:lb_non_convex}]
We define the distribution $\Z = \Z(\delta)$ over the set of datapoints $Z$ by
\begin{align*}
    \forall i \leq d; \quad 
    z(i) &= \begin{cases}
        -1 \quad \text{w.p. } \delta
        \\
        1  \quad \text{w.p. } \delta
        \\
        0 \quad \text{w.p. } (1-2\delta),
    \end{cases}
    \\
    \forall i > d ; \quad 
    z(i) &= \Ind{z(i-d) = -1}.
\end{align*}
We consider the same loss function of \cref{thm:lb_main}, but leave the norm-like component without the square-root;
\begin{align*}
    f(w; z) \eqq
    \phi(w; z)
    + \nu_z(w),
    \quad 
    \nu_z(w) \eqq 
    \sum_{i=1}^d z(i)w(i)^2,
\end{align*}
where $\phi$ is defined as in \cref{eq:main_phi}. We also define the gradient oracle for $\phi$ as we have done in the convex case \cref{eq:main_grad}, \cref{eq:main_grad_I}, and \cref{eq:main_grad_i_t}, repeated here for convenience;
\begin{align*}
    g_\phi(w; z)(i) &\eqq 
    \begin{cases}
        \Ind{i = I(w)} \quad & i \leq d
        \\
        - \epsilon z(i) + \Ind{i = I(w)+d}  \quad & i \geq d
    \end{cases},
    \\
    I(w) &\eqq 
    \min\Big\{ i \in [d] \mid 
        i \in \argmax_{1 \leq j \leq d} 
        \{ w(j) + w(j+ d) \}
    \Big\},
    \\
    i_t &\eqq I(w_t).
\end{align*}
Here, unlike the construction in \cref{thm:lb_main} we need $\epsilon$ to depend on $\eta$, and set $\epsilon \eqq \eta/d$.
The next lemma establishes the SGD iterates end up ``overfitting'' the empirical objective, and follows from a proof that is essentially identical to \cref{lem:main_Z} and \cref{lem:lb_iterates}. The only difference is that here the training examples have $-1$ rather than $1$ in the critical coordinates.
\begin{lemma}
For $\delta = 1/4n^2$, we have with probability $\geq 1/2$ that for all $\tau \in [n]$ and $t < \tau$;
\begin{itemize}
    \item $w_{\tau}(i_t) = -\eta$, 
    \item $1 \leq s < t \implies z_s(i_t) = -1$, and
    \item $t \leq s \leq n \implies z_s(i_t) \leq 0$.
\end{itemize}
\end{lemma}
Thus, let $\tau \leq n+1$, and denote $\wbar w_{\tau} \eqq \frac{1}{n-\tau+2} \sum_{t=\tau}^{n+1} w_t$.
By a derivation identical to that of the convex case, 
we obtain for all $1 \leq t \leq n/2$, 
$\wbar w_{\tau} (i_t) \leq - \frac{\eta}{2}$.
Hence, for $s\leq n/4$,
\begin{align*}
    \nu_{z_s}(\wbar w_\tau) \leq - \sum_{t = n/4 + 1}^{n/2} \wbar w_\tau(i_t)^2 = -\frac{n \eta^2}{16},
\end{align*}
therefore
\begin{align*}
    \hF(\wbar w_{\tau}) 
    &= \frac{1}{n}\sum_{t=1}^n \phi(\wbar w_\tau; z_t)
        + \frac{1}{n}\sum_{t=1}^n \nu_{z_t}(\wbar w_\tau)
    \\
    &\leq 2 \epsilon\eta n -\frac{1}{n}\frac{n}{2}\frac{n \eta^2}{16}
    \\
    &\leq  -\frac{n \eta^2}{64}
    .
\end{align*}
To conclude the proof, we note that
\begin{align*}
    F(\wbar w_{\tau}) \geq -\epsilon \sum_{i=d+1}^{2 d} \delta (\eta \epsilon n)
    = - \epsilon^2 d \delta n \eta 
    \geq - \eta^2 / n,
\end{align*}
and the result follows.
\end{proof}
  
\subsection{SCO with strongly convex components}
\label{sec:proof:sc_lb_main}

\begin{proof}[of \cref{thm:sc_lb_main}]
    We first make the argument for an unbounded domain, so that no projections take place. Let
    \begin{align*}
        f(w; z) &\eqq
        \phi(w; z)
        + \nu_z(w)
        + \frac{\lambda}{2}\norm{w}^2,
    \end{align*}
    where $\phi$ and $\nu$ are defined by
    \begin{align}
        \phi(w; z) &\eqq - \epsilon \sum_{i=d+1}^{2d} z(i)w(i) 
            + \max_{1 \leq i\leq d} \cb{w(i) + \epsilon w(i + d)},
        \\
        \nu_z(w) &\eqq 
        \sqrt{ \sum_{i=1}^d z(i)w(i)^2 },
     \end{align}
     and $\epsilon \eqq 1/d$.
     These are essentially the same definitions as in our main construction \cref{eq:main_phi} and \cref{eq:main_nu}, but with an added $\epsilon$ factor inside the max component of $\phi$. This only makes our formal argument simpler, but otherwise does not make any significant difference.
     For the gradient oracle, we define
     \begin{align}
        g_\phi(w; z)(i) &\eqq 
        \begin{cases}
            \Ind{i = I(w)} \quad & i \leq d
            \\
            - \epsilon z(i) + \epsilon \Ind{i = I(w)+d}  \quad & i \geq d
        \end{cases},
        \\
        \text{where }
        I(w) &\eqq 
        \min\Big\{ i \in [d] \mid 
            i \in \argmax_{1 \leq j \leq d} 
            \{ w(j) + \epsilon w(j+ d) \}
        \Big\},
    \end{align}
    and again denote the index picked by $g$ on round $t\in[n]$ by
    \begin{align}
        i_t \eqq I(w_t).
    \end{align}
    We then set
    \begin{align*}
        g(w; z) 
        &\eqq g_\phi(w; z) + \nabla \nu_z(w) + \lambda w.
    \end{align*}
    Clearly, for all $w, z\in \R^{2d}$, $g(w; z) \in \partial_w f(w; z) $.
    Following a direct computation, we get that
    \begin{align*}
        g(w; z)(i) &= 
        \begin{cases}
            \Ind{i = I(w)} + \frac{z(i) w(i)}{\nu_z(w)}  + \lambda w(i) 
            \quad & i \leq d,
            \\
            \epsilon \Ind{i = I(w)+d} - \epsilon z(i) + \lambda w(i) \quad & i > d,
        \end{cases}
    \end{align*}
    where $I$ is defined in \cref{eq:main_grad_I}.
    Hence, the stochastic gradient steps $w_{t+1} \gets w_t - \eta_t g(w_t, z_t)$ are given by
    \begin{align}
        w_{t+1}(i) 
        &= \begin{cases}
            \b{1 - \eta_t ( z_t(i) \nu_{z_t}(w)^{-1}  + \lambda) } w_t(i)
            -\eta_t \Ind{i = I(w)}
                \quad & i \leq d
            \\
            \b{1 - \eta_t \lambda } w_t(i) + \epsilon \eta_t z_t(i)
            -\eta_t \epsilon \Ind{i = I(w)+d}
                \quad & i > d
            .
        \end{cases}
        \label{eq:lbsc_gdsteps}
    \end{align}
    The next lemma makes a similar assertion as \cref{lem:lb_iterates} and follows from similar arguments.
    \begin{lemma}\label{lem:sc_iterates}
      With probability $\geq 1/2$, for all $\tau \in [n]$ we have 
    \begin{align*}
        w_{\tau+1}(i) = \begin{cases}
            -\eta_t \prod_{s=t+1}^\tau (1-\eta_s\lambda)
                \quad & i = i_t, t\in[\tau],
            \\
            0 \quad & i \in [d] \setminus \cb{i_1, \ldots i_\tau},
        \end{cases}
    \end{align*}
    where 
    $z_1(i_t) = \ldots = z_{t-1}(i_t) = 1$, and 
    $z_t(i_t) = \ldots = z_n(i_t) = 0$.
    \end{lemma}
    Next, a simple derivation shows the empirical risk is large for any suffix average. 
    Thus, let $\tau \leq n+1$, and denote $\wbar w_{\tau} \eqq \frac{1}{n-\tau+2} \sum_{t=\tau}^{n+1} w_t$.
    By \cref{lem:sc_iterates}, assuming the event from the lemma occurs we have
    \begin{align*}
        t \leq n/2 \implies \av{ \wbar w_{\tau} (i_t)}
        \geq \frac{1}{2}\av{ w_{n+1}(i_t) } = 
            \frac{\eta_t}{2}\prod_{s=t+1}^n (1-\eta_s\lambda).
    \end{align*}
    Therefore, noting that $\phi(\wbar w_{\tau}; z_s) \geq -\epsilon^2 d n / \lambda \geq -\epsilon n / \lambda \geq -1 / (\lambda n) $, by \cref{lem:loss_lb_util} we obtain;
    \begin{align*}
        \hF(\wbar w_{\tau:n}) 
        &\geq \frac{\lambda}{2}\norm{\wbar w_{\tau:n}}^2 
            - \frac{1 }{\lambda n}
            + \frac{1}{10 \sqrt n}\sum_{t=n/4}^{n/2} 
                \eta_t \prod_{s=t+1}^n (1-\eta_s\lambda)
        \\
        &\geq \frac{\lambda}{2}\norm{\wbar w_{\tau:n}}^2 
            - \frac{1}{\lambda n}
            + \frac{1}{40\lambda \sqrt n}
        \tag{for $\eta_t=1/\lambda t$}
        .
    \end{align*}
    Noting that $\hF(0) = 0$, we obtain the claim on the optimization error.
    For the generalization gap, first note that for any $t\leq n+1$,
    \begin{align*}
        \E_z \phi(w_t; z)
        &\leq - \epsilon \sum_{i=d+1}^{2d} \delta w_t(i) 
            + \epsilon n /\lambda 
        \\
        &\leq \epsilon \delta n /\lambda
        + \epsilon n /\lambda 
        \\
        &\leq 2 \epsilon n /\lambda,
    \end{align*}
    and observe;
    \begin{align*}
        F(\wbar w_{\tau}) 
        &\leq 
            \frac{2 \epsilon n}{\lambda}
            + \frac{\lambda}{2}\norm{\wbar w_{\tau}}^2 
            + \E_{z\sim \Z}\sb{ 
                \sqrt{\sum_{t=1}^n z(i_t) \wbar w_{\tau}(i_t)^2}
            }
        \\
        &\leq \frac{2 \epsilon n}{\lambda}
        + \frac{\lambda}{2}\norm{\wbar w_{\tau}}^2 
        + \sqrt{ \delta
            \sum_{t=1}^n \wbar w_{\tau}(i_t)^2
        }
        \tag{by Jensen's inequality}
        \\
        &\leq 
        \frac{2 \epsilon n}{\lambda}
        + \frac{\lambda}{2}\norm{\wbar w_{\tau}}^2 
        + \frac{1}{2n}
            \sum_{t=1}^n \av{\wbar w_{\tau}(i_t)}
        \\
        &\leq \frac{2 \epsilon n}{\lambda}
        + \frac{\lambda}{2}\norm{\wbar w_{\tau}}^2 
        + \frac{\log n}{2\lambda n}
        \tag{for $\eta_t=1/\lambda t$} 
        \\
        &\leq 
        \frac{\lambda}{2}\norm{\wbar w_{\tau}}^2 
        + \frac{\log n}{\lambda n}
        \tag{since $2 \epsilon n / \lambda \leq \log n/ (2\lambda n)$} 
        .
    \end{align*}
    Combining the inequalities in the last two displays completes the proof.
    
    \paragraph{Bounded domain case with $\lambda \geq 1/\sqrt n$.} In this case projections happen, but owed to our assumption on $\lambda$ we will see their effect is negligible.
    Denote 
    $w_s' \eqq w_s - \eta g(w_s, z_s)$ so that now $w_{s+1} \gets \Pi_W\b {w_s'}$.
    % Let $g(w_t; z_t) = \zeta_i + e_{i_t} + e_{(i_t + d)}$, 
    % and observe
    % \begin{align*}
    %     \norm{w_{t+1}'}^2
    %     &= \norm{w_t - \eta \zeta_t - \eta e_{i_t} - \eta e_{(t_t + d)}}^2
    %     \\
    %     &\leq 1 + d n\eta^2\epsilon^2 + 2 \eta^2
    %     \\
    %     &\leq 1 + (n/d)\eta^2 + 2 \eta^2
    %     \\
    %     &\leq 1 + 3 \eta^2
    %     .
    % \end{align*}
    Observe that by \cref{eq:lbsc_gdsteps} under the event of \cref{lem:sc_iterates}, we have
    \begin{align*}
        w_{\tau+1}'(i) = \begin{cases}
            0 
            \quad &i\in[d]\setminus\{i_1, \ldots, i_\tau\},
            \\
            (1 - \eta_\tau \lambda) w_{\tau}(i)
            + \eta_\tau \Ind{i = i_\tau}
            \quad &i\in \{i_1, \ldots, i_\tau\},
            \\
            (1 - \eta_\tau \lambda) w_{\tau}(i)
            + \epsilon \eta_\tau z_t(i) (-1 + \Ind{i = i_\tau + d})
            \quad &i\in[d+1, 2d], 
        \end{cases}
    \end{align*}
    thus
    \begin{align*}
        w_{\tau+1}'(i)^2 \leq \begin{cases}
            0 
            \quad &i\in[d]\setminus\{i_1, \ldots, i_\tau\},
            \\
            (1 - \eta_\tau \lambda)^2 w_{\tau}(i)^2
            \quad &i\in \{i_1, \ldots, i_{\tau-1}\},
            \\
            \eta_\tau^2
            \quad &i = i_\tau,
            \\
            (n \epsilon /\lambda)^2
            \quad &i\in[d+1, 2d].
        \end{cases}
    \end{align*}
    Note that by our assumption that $\lambda \geq 1/\sqrt n$ and our choice of $\epsilon=1/d$, we have $n^2\epsilon^2/\lambda^2 \leq n/d^2$. 
    Now, fix $\tau\geq 2 n/3$, and observe;
    \begin{align*}
        \norm{w_{\tau+1}'}^2 
        &\leq \eta_\tau^2 + d(n/d^2)
            + (1 - \eta_{\tau}\lambda)^2 \sum_{i=1}^{2d} w_{\tau}(i)^2
        \\
        &= \eta_\tau^2 + n/d
            + (1 - \eta_{\tau}\lambda)^2 \norm{w_\tau}^2
        \\
        &\leq \eta_\tau^2 
            + n/d
            + (1 - \eta_{\tau}\lambda)^2
        \\
        &= \frac{1}{\lambda^2 \tau^2}
            + \frac{n}{d}
            - \frac{2}{\tau}
            + \frac{1}{\tau^2}
        \\
        &\leq \frac{2}{\lambda^2 n \tau}
            + 1 
            + \frac{n}{d}
            + \frac{9}{4 n^2}
            - \frac{2}{\tau}
        \\
        &\leq \frac{3}{ 2\tau}
            + 1 
            + \frac{10}{4 n^2}
            - \frac{2}{\tau}
        \\
        &\leq 1.
    \end{align*}
    In the above, we have used that $n/d \leq 1/4n^2$, and that $1/2\tau \geq 1/(2 n) \geq 10/4n^2$ for $n\geq 10$.
    Hence, from round $2 n/3$ onwards projections do not occur anymore.
    To conclude the proof, we note that we can lower bound the empirical loss precisely as we did before but over rounds $4 n/6$ to $5n/6$, rather than $n/4$ to $n/2$. In addition, the population loss has only improved since the per coordinate values in all iterates have only decreased in magnitude as a result of the projections.
    \end{proof}
    
    \begin{proof}[of \cref{lem:sc_iterates}]
        Note that for all $t\in[n]$ and any $i \in [d+1, 2d] \setminus \{i_1+d, \ldots, i_{t}+d\}$, by \cref{eq:lbsc_gdsteps} we have
        \begin{align}
            w_{t+1}(i)
            &= \b{1 - \eta_t \lambda } w_t(i) + \epsilon \eta_t z_t(i)
            \nonumber \\
            &= \sum_{s=1}^t 
                \epsilon \eta_s z_s(i)
                \prod_{l=s+1}^t \b{1 - \eta_l \lambda }
            .
            \label{eq:sc_iter_expression}
        \end{align}
        Hence, the number of times $z_s(i) = 1$ for $s \leq t$ determines the maximality of $w_{t+1}(i)$. In other words, the extra component in the gradient update effects all coordinates equally, and the situation here is no different than the convex case.
        Thus, by \cref{lem:main_Z} and the same arguments as given in \cref{lem:lb_iterates}, we have that with probability $\geq 1/2$, for all $t\in[n]$, $i_t = I(w_t) = \io_t$. Therefore, by \cref{eq:lbsc_gdsteps};
        \begin{align*}
            w_{t+1}(i) = \begin{cases}
                -\eta_t 
                    \quad & i = i_t;
                \\
                -\eta_s \prod_{l=s+1}^t (1-\eta_l\lambda)
                    \quad & i = i_s,\; s < t;
                \\
                0 \quad & i \in [d] \setminus \cb{i_1, \ldots, i_t}.
            \end{cases}
        \end{align*}
        Note that the $z_t(i)\nu_{z_t}(w)^{-1}$ component in \cref{eq:lbsc_gdsteps} does not contribute since $z_t(i_s) = 0$ for all $t \geq s$, given our event.
        \end{proof}

\section[Proofs for with-replacement SGD]{Proofs for \cref{sec:wrsgd}}
\label{sec:proof:wrsgd_std_ub}

% \paragraph{SGD analysis with weighted averaging}
% \subsection{SGD analysis with weighted averaging}
% \label{sec:proof:wrsgd_std}
In what follows we provide the standard analysis of SGD with the iterate averaging scheme specified in \cref{thm:wrsgd_population_ub}.
The theorem stated and proved below provides the rate of convergence on the target objective function from which gradients are sampled (as similar analyses normally do); note that we use it in the context where the target objective is the \emph{empirical} loss given by the training set. This should be contrasted with the goal of \cref{thm:wrsgd_population_ub}, which is to establish the convergence rate on the \emph{population} objective. Our only motivation for proving the below theorem is to argue the generalization gap upper bound established in \cref{cor:wrsgd_gengap}.

  \begin{theorem}
    Let $W\subset \R^d$ with diameter $D$, and $f_1, \ldots, f_n$ be a sequence of convex, $G$-Lipschitz losses sampled i.i.d.~from some distribution $\mathcal F$.
    Further, let $w^\star \eqq \min_{w\in W} \E_{f \sim \mathcal F} f(w) $ denote the minimizer of the expected function.
    Then, the weighted average $\wbar w \eqq \frac{2}{n+1} \sum_{t=1}^n \frac{n-t+1}{n} w_t$ of the iterates produced by SGD with step size $\eta=\frac{D}{G\sqrt n}$ obtains the following upper bound:
    \begin{align*}
        \E_{f_1, \ldots, f_n, f \sim \mathcal F}\sb{f(\wbar w) - f(w^\star)}
        \leq \frac{4 G D}{\sqrt n}
        .
    \end{align*}
  \end{theorem}
  \begin{proof}
    Observe,
    \begin{align}
        \E \sb{f(\wbar w) - f(w^\star)}
        &\leq \frac{2}{n+1} \sum_{t=1}^n \frac{n-t+1}{n} \E \sb{f(w_t) - f(w^\star)}
        \nonumber \\
        &= \frac{2}{n+1} \E \sb{ \sum_{t=1}^n \frac{n-t+1}{n} \b{ f_t(w_t) - f_t(w^\star) }}
        .
        \label{eq:wrsgd_emp_1}
    \end{align}
    By the standard SGD analysis,
    \begin{align*}
        f_t(w_t) - f_t(w^\star)
        \leq \nabla f_t(w_t)\T (w_t - w^\star)
        \leq \frac{1}{2\eta}\b{D_t^2 - D_{t+1}^2} + \frac{\eta}{2}G^2,
    \end{align*}
    where $D_t \eqq \norm{w_t - w^\star}$.
    Now,
    \begin{align*}
        \sum_{t=1}^n \frac{n-t+1}{n} \b{ f_t(w_t) - f_t(w^\star) }
        &\leq \frac{1}{2\eta} \sum_{t=1}^n \frac{n-t+1}{n} \b{D_t^2 - D_{t+1}^2}
        + \frac{\eta G^2 n}{2}
        \\
        &= \frac{D_1^2}{2\eta} 
            + \frac{1}{2\eta}\sum_{t=2}^n D_t^2 \b{\frac{n-t+1}{n} - \frac{n-t}{n} }
            + \frac{\eta G^2 n}{2}
        \\
        &= \frac{D_1^2}{2\eta} 
            + \frac{1}{2\eta n} \sum_{t=2}^n D_t^2
            + \frac{\eta G^2 n}{2}
        \\
        &\leq \frac{D^2}{\eta} 
        + \frac{\eta G^2 n}{2}
        \\
        &\leq 2 G D \sqrt n
        .
    \end{align*}
    Plugging the above inequality into \cref{eq:wrsgd_emp_1}, we obtain
    \begin{align*}
        \E \sb{f(\wbar w) - f(w^\star)}
        \leq \frac{4 G D}{\sqrt n}
        ,
    \end{align*}
    which completes the proof.
  \end{proof}

\section[Proofs - Multiepoch]{Proofs for \cref{sec:multiepoch}}

\subsection{Lower bound for multi-epoch SGD}
\label{sec:proof:melb_main}

\begin{proof}[of \cref{thm:melb_main}]
    First, note that without modifications, the strategy of \cref{thm:lb_main} breaks after the first epoch; it will just keep pointing the gradient on coordinates with an all ones sequence.
    We use the idea we can fully ``record'' into the iterate $w_t^k$ the precise samples we have stepped through so far, and define a gradient oracle that will cause the iterate to advance on fresh bad coordinates in every new epoch.
    We will work with the datapoints set $Z=\{0, 1\}^{d}$ and define the distribution $\Z = \Z(\delta)$ over $Z$ by letting $z(i) \sim \mathrm{Ber}(\delta)$ for all $i\in[d]$.
    % \begin{align*}
    %     \forall i \in [d];\;
    %     z(i) &= \begin{cases}
    %         1 \quad \text{w.p. } \delta ;
    %         \\
    %         0 \quad \text{w.p. } (1-\delta),
    %     \end{cases}
    % \end{align*}
    We consider two separate portions of a vector $w \in \R^{d'}$, which we denote by $w[\cdot; 0] \in \R^d$ and $w[\cdot, \cdot; 1] \in \R^{d \times n K} $.
    The first portion with $d$ entries is where the bad gradient steps will be made and where we will eventually suffer the loss from.
    The second consists of $d n K$ entries and is used to encode the samples observed during the optimization process.
    Our loss function is defined as follows;
    \begin{align}
    \nu_z(w) &\eqq \sqrt{\sum_{i=1}^d z(i) w[i; 0]^2},
    \nonumber
    \\
    \phi(w; z) &\eqq \epsilon \sum_{i=1}^d 
        (1+z(i))\max_{j\in[n K]} \{ w[i, j; 1] \} + \max_{i\in[d]} \{ w[i; 0]) \},
    \label{eq:multiepoch_lb_phi}
    \\
    f(w; z) &\eqq \phi(w; z) + \nu_z(w).
    \nonumber
    \end{align}
    Again, we choose $\epsilon > 0$ sufficiently small so that the loss induced by it is negligible, and so that $f$ is $4$-Lipschitz.
    The gradient oracle we use is specified by that of the $\phi$ function;
    \begin{align*}
        g_\phi(w; z)[i, j; 1]
        &= \epsilon (1+z(i))
            \Ind{j = I_1(w[i, \cdot; 1])} 
        \\
        g_\phi(w; z)[\cdot; 0] &= e_{I_0(w)} 
        \\ \text{where }
        I_1(x) &\eqq \min \cb{ j\in [n K] 
                \;\Big|\;
                j\in\argmax_{l\in[n K]} \cb{ x(l) }
            }\; \text{ for } x\in \R^{n K},
        \\ \text{ }
        I_0(w) &\eqq  \min \cb{ i\in [d] 
            \;\Big|\;
            i\in\argmax_{l\in[d]} \cb{ w[l; 0] },
            \text{ and }
            i\in\argmin_{l\in[d]} \cb[B]{ 
                V(w)(l)
            }
        },
        \\ \text{ and }
        V(w) &\eqq \sum_{j = t_0(w)}^{n K} w[\cdot, j; 1],
    \end{align*}
    where $t_0(w)$ denotes the first global iteration index of the current epoch. This index is easy to infer from $w[\cdot, \cdot; 1]$ since the $t$'th SGD iteration in epoch $k$ results in values strictly smaller than $0$ in all entries of $w[\cdot, \tau; 1]$, where $\tau = n(k-1) + t$ (and $0$ remains from initialization for $\tau' > \tau$).
    In words, we design our ``adversarial'' gradient oracle so that it will choose the coordinate for $\max_{i\in[d]} \{ w[i; 0] \}$, by ``looking'' in 
    $w[\cdot, \cdot; 1]$ and choosing the coordinate $i_t^k \eqq I_0(w_t^k) \in[d]$ such that the number if times $z_s^k(i_t^k) = 1$ for $s < t$ is largest.
    An illustration is provided in \cref{fig:melb}.

    \begin{figure}[!ht]
        \center
        \includegraphics[width=14cm]{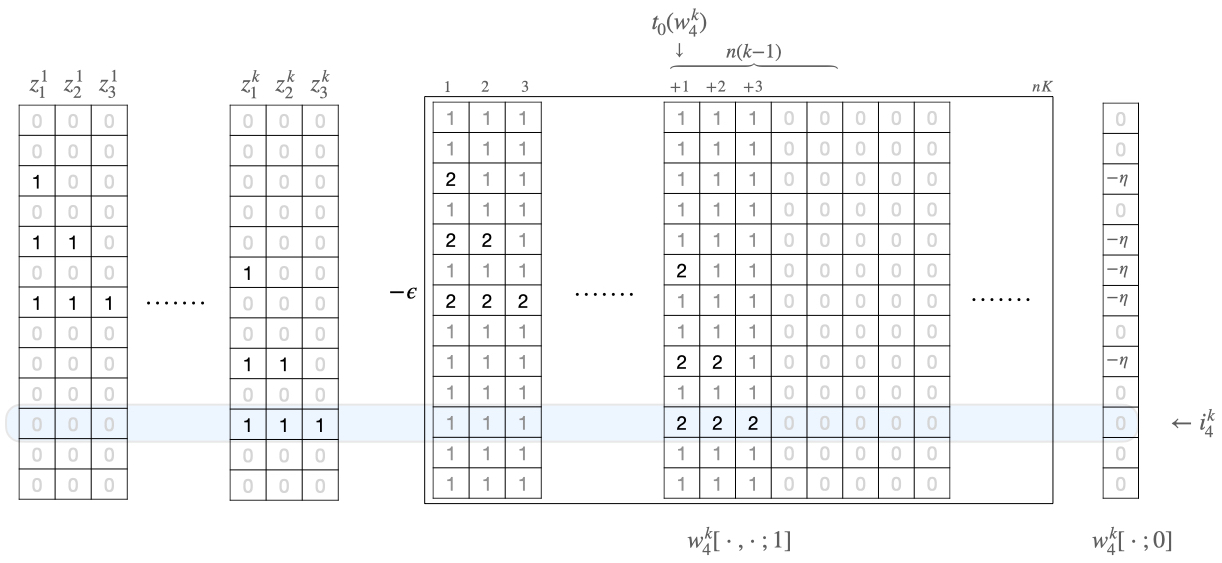}
        \caption{Illustration of gradient oracle mechanism \label{fig:melb}}
    \end{figure}

    In similar spirit to the the basic construction from \cref{thm:lb_main}, we will ensure that with high probability, the coordinates selected by our gradient oracle are such that $z_s^k(i_t^k) = 1$ for all $s < t$, and $z_s^k(i_t^k) = 0$ for all $s \geq t$.
    To that end, we first assert the existence of a set of datapoints $Z \subset \cb{0, 1}^d$ where a desired property described next holds with sufficiently high probability.
    Consider some arbitrary ordered set $S =\cb{z_1, \ldots, z_n} \subset \cb{0, 1}^d$.
    For $t\in[n]$, denote
    \begin{align*}
        E^{(1,t)} &\eqq \big\{i\in [d] 
            \mid z_s(i) = 1 \; \forall s < t
        \big\},
        \\
        E^{(1,t; K)} &\eqq 
        \cb{\io_t^1, \ldots, \io_t^{K} \mid 
        \io_t^k \text{ is the k'th smallest} \in E^{(1,t)} },
        \\
        E^{(t, 0)} &\eqq \big\{i\in [d] 
            \mid z_s(i) = 0 \; \forall s \geq t
        \big\}.
    \end{align*}
    So $E^{(1,t; K)}$ is just the first $K$ elements of $E^{(1, t)}$, where we enumerate the coordinates by the superscript in increasing order. 
    % Note that in general $\io_t^k \neq i_t^k$ (in fact, this will rarely hold for multi-shuffle).
    We say the event $\mathcal E$ holds for the set $S$, or equivalently that $S \in \mathcal E$ if for all $t\in [n]$, we have $|E^{(1,t; K)}| = K$, and 
    $E^{(1,t; K)} \subseteq E^{(t, 0)}$. In words, $S \in \mathcal E$ if for every $t\in [n]$, the first $K$ coordinates $\cb{\io_t^1, \ldots, \io_t^K}$ that have a prefix of $(t-1)$ ones; $s < t \implies z_s(\io_t^k) = 1$, also satisfy that they have a suffix of zeros; $s \geq t \implies z_s(\io_t^k) = 0$.
    \begin{lemma}\label{lem:multiepoch_lb_main}
        There exists a set of datapoints $Z = \cb{\zeta_1, \ldots, \zeta_n} \subset \cb{0, 1}^d$, such that
        \begin{align*}
            \Pr_{\pi_1, \ldots, \pi_K \sim \Pi([n])}
                \b{\forall k\in [K],\; \cb{z_1^k, \ldots, z_n^k} \in \mathcal E}
                \geq 1/2,
        \end{align*}
        where $\pi_k$ is sampled by either single-shuffle or multi-shuffle, and $z_t^k \eqq \zeta_{\pi_k(t)}$. 
    \end{lemma}
    With \cref{lem:multiepoch_lb_main} in place, we can be sure bad coordinate sets will turn up in every epoch.
    Let $k \in [K]$ and $t\in[n]$, and assume the event from the lemma occurs.
    By the definition of our gradient oracle, it is bound to select one of the first $k$ coordinates that have a prefix of all ones, which we are assured by the lemma will also have a suffix of zeros.
    Formally, we argue by induction on $k, t$. The base case follows from the definition of $g_\phi$ and our assumption that the event occurs.
    For the inductive step, assume the selected coordinates $i_{t'}^{k'}$ of all prior rounds satisfy the inductive hypothesis. Then at most $k-1$ of the first coordinates in $E^{(1,t, K)}$ could have been selected previously, since the inductive hypothesis implies every selected coordinate $i_{t'}^{k'}$ has exactly $t'-1$ ones. Hence, the gradient oracle will choose a coordinate from $E^{(1,t, K)}$, the elements of which are coordinates that also enjoy a suffix of zeros, as assured by the event from the lemma.
    In addition, note that our initialization at $w_1^1 = 0$ and our assumption that $\eta \leq 1/\sqrt{2 n K}$ (and that $\epsilon$ is negligibly small) ensure the iterate never leaves the domain $W$ thus no projections occur.
    Summarizing, we have that for every $k \in [K]$, $t\leq n+1$, it holds that:
    \begin{align*}
        t < t' &\implies 
        w_{t'}^k(i_t^k) = -\eta
        \\
        \text{and }
        s < t 
        &\implies 
        z_s^k(i_t^k) = 1.
    \end{align*}
    To complete the proof, we will now prove a lower bound of $\Omega(\eta \sqrt{n/J})$ for the average iterate of the last $J$ epochs.
    The other terms in the bounds of the theorem statement follow from concatenating our problem instance dimension-wise with standard constructions --- see \cref{lem:std_convex_lb}.
    % Proceeding, note that for every epoch $k\in [K]$, 
    % \begin{align*}
    %     f(w_{n+1}^k; z_s^k) 
    %     \gtrapprox \sqrt{n-s} \eta
    %     \implies F(w_{n+1}^k) 
    %     \gtrapprox \eta \sqrt n,
    % \end{align*}
    % and moreover the average iterate of any epoch is $\eta\sqrt n$ sub-optimal by arguments similar to those of the single epoch case \cref{thm:lb_main}.
    Proceeding, we slightly overload notation and denote $w(i)$ for $w[i, 0]$.
    Let $\hat w$ be the average of the iterates in the last $J \in [K]$ epochs;
    \begin{align*}
        \hat w \eqq \frac{1}{n J} \sum_{k=K-J+1}^K \sum_{t=1}^{n+1} w_t^k
        .
    \end{align*}
    For all $n/4 \leq t \leq 3n/4$ and $K-J+1 \leq k \leq K$, we have
    \begin{align}
        \av{\hat w (i_t^k)} \geq \frac{1}{n J}
            \sum_{t'=3n/4}^{n} \av{ w_{t'}^k (i_t^k) }
        \geq \frac{\eta}{n J} (n/4) = \frac{\eta}{4 J},
        \label{eq:melb_avg_iter}
    \end{align}
    since $w_{t'}^k(i_t^k) = -\eta$ for all $t' \geq t$.
    \paragraph{Single-shuffle case.}
    Ignoring the negligible $\epsilon$ terms, we now have
    \begin{align*}
        1 \leq s \leq n/4
        \implies 
        f(\hat w; z_s^1)
        &\geq \sqrt{\sum_{k=1}^K \sum_{t=s}^{n/2} 
            z_s^1(i_t^k) \hat w(i_t^k)^2}
        \\
        &\geq  \sqrt{\sum_{k=K-J+1}^K \sum_{t=n/4+1}^{n/2} 
            z_s^1(i_t^k) \hat w(i_t^k)^2}
        = \sqrt{(n J/4) \frac{\eta^2}{16 J^2}}
        = \sqrt{\frac{n}{64 J}} \eta,
    \end{align*}
    since $z_s^k(i_t^k) = z_s^1(i_t^k) = 1 \; \forall s < t, \;k\in[K]$. Therefore
    \begin{align*}
        F(\hat w) &= \frac{1}{n} \sum_{s=1}^n f(\hat w; z_s^1)
        \geq \frac{n/4}{n} \sqrt{\frac{n}{64 J}} \eta
        \geq  \sqrt{\frac{n}{J}} \frac{\eta}{32},
    \end{align*}
    concluding the proof for this case.
    \paragraph{Multi-shuffle case.} 
    For $\zeta_i \in Z$ denote 
    \begin{align*}
        I^k(\zeta_i) &\eqq \Ind{z_t = \zeta_i, t\in[n/4, 3n/4]},
        \\
        I(\zeta_i) &\eqq \sum_{k=K-J+1}^K I^k(\zeta_i),
        \\
        \tilde Z &\eqq \cb{ \zeta_i \in Z \mid I(\zeta_i) \geq J/4}.
    \end{align*}
    We wish to lower bound the size of $\tilde Z$, to show that enough $\zeta_i$'s where incident in the $[n/4, 3n/4]$ iteration range in a sufficiently large number of epochs. 
    (Our interest in this range stems from the desire to apply \cref{eq:melb_avg_iter}.)
    Observe;
    \begin{align*}
        \sum_{k=K-J+1}^K \sum_{i=1}^n I^k(\zeta_i) 
        &= J \frac{n}{2}
        \\
        \implies
        \sum_{i=1}^n I(\zeta_i) 
        &= \frac{n J}{2}.
    \end{align*}
    Since $I(\zeta_i) \leq J$ for all $i\in[n]$, by the pegionhole principle we get that $\av[b]{ \tilde Z } \geq n/4$. Otherwise, we would have
    \begin{align*}
        \sum_{i=1}^n I(\zeta_i) 
        \leq \av[b]{ \tilde Z } J + \b[b]{n - \av[b]{\tilde Z} }\frac{J}{4}
        < J\frac{n}{4} + \frac{J}{4}n
        = \frac{n J}{2}.
    \end{align*}
    Now,
    \begin{align*}
        F(\hat w) 
        \geq \frac{1}{n} \sum_{z \in \tilde Z} f(\hat w; z)
        &\geq \frac{1}{n} \sum_{z \in \tilde Z} 
            \sqrt{\sum_{k=K-J+1} I^k(z) \sum_{t=3n/4}^n \hat w(i_t^k)^2}
        \\
        &\geq \frac{1}{n} \sum_{z \in \tilde Z} 
            \sqrt{\frac{J}{4} \sum_{t=3n/4}^n \frac{\eta^2}{16 J^2}}
        \\
        &= \frac{1}{n} \sum_{z \in \tilde Z} 
            \sqrt{\frac{n \eta^2}{16^2 J}}
        \\
        &\geq \frac{ \eta \sqrt n}{64 \sqrt J},
    \end{align*}
    which concludes the multi-shuffle case and the proof as a whole.
    \end{proof}

    \begin{proof}[of \cref{lem:multiepoch_lb_main}]
        We will make our argument for an i.i.d.~sampled instance set $Z$, and convert it to the stated result as follows. Assume $\Z$ is a distribution over $\cb{0, 1}^d$ for which we establish the following;
        \begin{align}
            \Pr_{Z \sim \Z^n}(Z \in \mathcal E) \geq \frac{1}{2^{1/K}}.
            \label{eq:melb_1}
        \end{align}
        Clearly, applying a random permuation on $Z \sim \Z^n$ does not change its distribution, therefore
        \begin{align*}
            \frac{1}{2^{1/K}}
            \leq \mathrm{Pr}_{Z \sim \Z^n}(Z \in \mathcal E)
            &= \mathrm{Pr}_{Z \sim \Z^n, \pi \sim \Pi([n])}(\pi(Z) \in \mathcal E)
            \\
            &= \sum_{Z \in Z^n} 
                \Pr_{\Z^n}(Z)
                \mathrm{Pr}_{\pi \sim \Pi([n])}(\pi(Z) \in \mathcal E)
            \\
            &\leq \max_{Z \in Z^n} \cb{ \mathrm{Pr}_{\pi \sim \Pi([n])}(\pi(Z) \in \mathcal E) }.
        \end{align*}
        The above derivation implies the existence of $Z^\star \in Z^n$ with the property that
        \begin{align*}
            \mathrm{Pr}_{\pi_1, \ldots, \pi_K \sim \Pi([n])}
                (\forall k\leq K, \; \pi_k(Z^\star) \in \mathcal E) 
            &\geq \b{ \frac{1}{2^{1/K}} }^K = \frac{1}{2},
            \tag{multi-shuffle}
            \\
            \text{ and }
            \mathrm{Pr}_{\pi_1 \sim \Pi([n])}
                (\forall k\leq K, \; \pi_k(Z^\star) \in \mathcal E) 
            &\geq \frac{1}{2^{1/K}}\geq \frac{1}{2},
            \; \text{ where } \pi_1 = \ldots = \pi_K
            \tag{single-shuffle}
            .
        \end{align*}
        Therefore, for the rest of the proof we focus on proving the distribution $\Z$ as defined next satisfies the desired property \cref{eq:melb_1}.
        Let $\delta > 0$ which will be chosen in hindsight, and consider $\Z = \Z(\delta)$ where $z(i) \sim \mathrm{Ber}(\delta)$ for each $i\in [d]$ independently.
        % \begin{align*}
        %     z(i) &= \begin{cases}
        %         1 \quad \text{w.p. } \delta,
        %         \\
        %         0 \quad \text{w.p. } (1-\delta).
        %     \end{cases}
        % \end{align*}
        Fix $t\in [n]$, and let $\mathcal E_t$ denote the event that 
        $|E^{(1,t)}| \leq K$, and $E^{(1, t; K)} \subseteq E^{(t, 0)}$.
        We will prove $\mathcal E_t$ holds with sufficiently high probability, so that $\mathcal E = \cap_{t\in[n]} \mathcal E_t$ holds w.p.~$\geq 1/2^{1/K}$.
        Proceeding, assuming we choose $\delta$ and $d$ so that $K < d \delta^n$, by Hoeffding's inequality we have that
        \begin{align}
            \Pr\b{ |E^{(1,t)}| \leq K} 
            &=
            \Pr\b{ \sum_{i=1}^d \Ind{i \in E^{(1,t)}} \leq K}
            \nonumber \\
            &=
            \Pr\b{ \sum_{i=1}^d \Ind{i \in E^{(1,t)}} 
                -d \delta^{t-1}
                \leq K -d \delta^{t-1}}
            \nonumber \\
            &= \Pr\b{ d \delta^{t-1} - \sum_{i=1}^d \Ind{i \in E^{(1,t)}} 
            \geq d \delta^{t-1} - K}
            \nonumber \\
            &\leq e^{-(d\delta^{t-1} - K)^2/ d}
            \leq e^{-(d\delta^n - K)^2/ d}
            .
            \label{eq:melb_lem_size}
        \end{align}
        In addition, for $i \in \cb{\io_t^1, \ldots, \io_t^K} = E^{(1, t; K)}$, we have 
        \begin{align*}
            \Pr(i \in E^{(t, 0)})
            = \Pr(\forall s \geq t, \; z_s(i) = 0) 
            \geq (1 - \delta)^n.
        \end{align*}
        Therefore, 
        \begin{align*}
            \Pr(E^{(1, t; K)} \subseteq E^{(t, 0)}) \geq \b{1 - \delta}^{nK}.
        \end{align*}
        From the above and \cref{eq:melb_lem_size} we obtain
        \begin{align*}
            \Pr(\text{not } \mathcal E_t) =
            \Pr \b{ E^{(1, t; K)} \not\subseteq E^{(t, 0)} \text{ or } |E^{(1,t)}| < K}
            \leq 
                e^{-(d\delta^n - K)^2/d} 
                + 1 - (1 - \delta)^{n K},
        \end{align*}
        hence,
        \begin{align}
            \Pr\b{ \cap_{t\in[n]} \mathcal E_t }
            \geq 1 - n (e^{-(d\delta^n - K)^2/d} + 1 - (1 - \delta)^{n K}).
            \label{eq:melb_Et_bound}
        \end{align}
        To finish the proof, we choose $\delta$ and $d$ as follows.
        Set $\delta \eqq 1/c n^2 K$, and note that 
        \begin{align*}
            (1-\delta)^{nK}
            =
            \b{ 1 - \frac{1}{c n^2K} }^{nK}
            \geq 1 - \frac{1}{c n}
            \implies 
            1 - (1 - \delta)^{n K}
            \leq \frac{1}{c n }.
        \end{align*}
        In addition, note that
        \begin{aligni*}
            \fraci{-(d\delta^n - K)^2}{d}
            \leq -d\delta^{2n} + 2K\delta^n,
        \end{aligni*}
        hence
        \begin{align*}
            e^{-(d\delta^n - K)^2/d} 
            &\leq \frac{1}{c n}
            \\
            \impliedby
            e^{-d\delta^{2n} + 2K\delta^n} 
            &\leq \frac{1}{c n}
            \\
            \iff \log(cn) + 2 K \delta^n 
            &\leq d\delta^{2n}
            \\
            \iff (\log(cn) + 2 K \delta^n) (c n^2 K)^{2n}
            &\leq d,
        \end{align*}
        which holds for any $ d \geq 2^{6n \log(c n K)} = (c n K)^{6n} \geq 2\log(c n) (c n^2 K)^{2n}$.
        Back to \cref{eq:melb_Et_bound} we obtain for $c=\frac{4}{2^{1/K}-1}$;
        \begin{align*}
            \Pr\b{ \cap_{t\in[n]} \mathcal E_t }
            \geq 1 - n\left(\frac{1}{c n} + \frac{1}{c n}\right)
            = 1 - \frac{2}{c} \geq \frac{1}{2^{1/K}},
        \end{align*}
        and we are done.
    \end{proof}
  
\subsection{Upper bound for single-shuffle multi-epoch SGD}
\label{sec:proof:mesgd_ub}
First, we slightly generalize the notion of uniform argument stability and prove some supporting lemmas.
We extend the definition of uniform-argument-stability \cref{eq:uas_A} to one that enables more than one difference in the training sets. We give the definition below in notation suitable for SGD and the lemmas that follow;
\begin{align}
    \epstabSGD (\tau; J) \eqq \max_{f_1, \ldots, f_\tau, f_1', \ldots, f_J'} 
    \max_{i_1, \ldots, i_J \in[\tau]}
    \normb{w_{\tau+1} - w_{\tau+1}'},
    \label{eq:sgd_uniform_stab_mult}
\end{align}
where $w_{\tau+1}'$ is the output of GD after swapping $f_1, \ldots, f_\tau$ in locations $i_1, \ldots, i_J$ with the other losses $f_1', \ldots, f_J'$. 
\cref{lem:stab_gen_single} given next generalizes \cref{lem:stab_gen} for the stability notion we have introduced above. The proof provided below is based on similar lemmas given in \cite{nagaraj2019sgd}.
\begin{lemma}\label{lem:stab_gen_single}
    Let $\cb{f(w; t)}_{t=1}^{n}$ be a set of $n$, $G$-Lipschitz losses, and $F(w) = \frac{1}{n} \sum_{t=1}^n f(w ; t)$.
    Then, for a uniformly random permutation $\pi\colon [n] \leftrightarrow [n]$, $f_t^k = f(\cdot; \pi(t)) \forall k$, it holds that for single-shuffle SGD;
    \begin{align*}
        \E \sb{ F(w_t^k) - f_t^k(w_t^k) }
        \leq G \epstabSGD(n(k-1) + t-1; 2k) ,
    \end{align*}
    where $w_t^k$ the $t$'th SGD iterate of the $k$'th epoch.
\end{lemma}
\begin{proof}
    Fix $t, i\in [n]$, and let $\pi(f_t \gets f(\cdot; i))$ denote the distribution obtained from a random permutation followed by replacing $f_t$ with $f(\cdot; i)$. In addition, denote by $\pi \mid f_t = f(\cdot; i)$ a uniformly distributed permutation conditioned on $f_t = f(\cdot; i)$. It is easily verified both distributions coincide. Now, by the law of total expectation;
    \begin{align*}
        \E_{f_1 \ldots f_n } \sb{ f_t(w_t^k) }
        &= \frac{1}{n} \sum_{i=1}^n 
            \E_{f_1 \ldots f_n} \sb{ f(w_t^k; i) 
                \mid f_t = f(\cdot; i)}
        \\
        &= \frac{1}{n} \sum_{i=1}^n 
            \E_{f_1 \ldots f_n \sim \pi \mid f_t = f(\cdot; i)} 
                \sb{ f(w_t^k; i)}
        \\
        &= \frac{1}{n} \sum_{i=1}^n 
            \E_{f_1 \ldots f_n \sim \pi(f_t \gets f(\cdot; i))} 
                \sb{ f(w_t^k; i)}
        \\ &= \frac{1}{n} \sum_{i=1}^n 
            \E_{f_1 \ldots f_n} \sb{ f(w_t^{k,(i)}; i) },
    \end{align*}
    where $w_t^{k, (i)}$ denotes the SGD iterate obtained for the datapoint sequence after replacing $f_t^j$ with $f(\cdot ; i)$ in all epochs $j\leq k$. Note this means each epoch differs from its original version in either $0$ or $2$ indexes.
    Now
    \begin{align*}
        \E_{f_1 \ldots f_n} 
            \sb{ F(w_t^k) - f_t^k(w_t^k) }
        &= \frac{1}{n}\sum_{i=1}^n 
            \E_{f_1 \ldots f_n} 
            \sb{ f(w_t^k; i) - f(w_t^{k, (i)}; i) }
        \\
        &\leq \max_{i\in [n]} G \norm{w_t^k - w_t^{k, (i)}}
        \\
        &\leq \epstabSGD(n(k-1) + t - 1; 2k)
        ,
    \end{align*}
    which completes the proof.
\end{proof}
We will also make use of a generalization of \cref{lem:sgd_stab} given in \cite[Lemma 3.1]{bassily2020stability}. The next lemma is a direct implication of it.
\begin{lemma}\label{lem:sgd_stab_mult}
    The generalized uniform-argument-stability (see \cref{eq:sgd_uniform_stab_mult}) rate of SGD with step-size $\eta > 0$ for $G$-Lipschitz convex functions satisfies
    \begin{align*}
        \epstabSGD(\tau; J) \leq 2 G\eta\sqrt {\tau} + 4 \eta G J.
    \end{align*}
\end{lemma}

We are now ready to prove the single-shuffle convergence upper bound.
  \begin{proof}[of \cref{thm:mesgd_ub} (single-shuffle case)]
    Similarly to the multi-shuffle case, we have;
    \begin{align*}
        F(\hat w) - F(w^\star)
        &\leq \frac{1}{n K}\sum_{k=1}^K \sum_{t=1}^n F(w_t^k) - F(w^\star)
        \\
        &= \frac{1}{n K}\sum_{k=1}^K \sum_{t=1}^n F(w_t^k) - f_t^k(w^\star)
        \\
        &= \frac{1}{n K}\sum_{k=1}^K \sum_{t=1}^n F(w_t^k) - f_t^k(w_t^k)
        + \frac{1}{n K}\sum_{k=1}^K \sum_{t=1}^n f_t^k(w_t^k) - f_t^k(w^\star)
        \\
        &\leq
        \frac{1}{n K}\sum_{k=1}^K \sum_{t=1}^n F(w_t^k) - f_t^k(w_t^k)
        + \frac{D^2}{ 2 \eta n K} + \frac{\eta G^2}{2},
    \end{align*}
    with the last inequality following from the standard $n K$ round regret bound for gradient descent \citep[see e.g.,][]{hazan2019introduction}. 
    To bound the other term, we now apply \cref{lem:stab_gen_single} and \cref{lem:sgd_stab_mult} to obtain;
    \begin{align*}
        \E \sb{ F(w_t^k) - f_t^k(w_t^k)) }
        &\leq G \epstabSGD(nk + t; 2k)
        \\
        &\leq 2\eta G^2 (\sqrt{n(k-1) + t} + 4 k )
        \\
        &\leq 2\eta G^2 (\sqrt{n K} + 4 K )
        .
    \end{align*}
    Now,
    \begin{align*}
        \E \sb{ F(\hat w) - F(w^\star) }
        &\leq
        \frac{1}{n K}\sum_{k=1}^K \sum_{t=1}^n\E \sb{  F(w_t^k) - f_t^k(w_t^k) }
        + \frac{D^2}{ 2 \eta n K} + \frac{\eta G^2}{2}
        \\
        &\leq
        \frac{1}{n K}\sum_{k=1}^K \sum_{t=1}^n ( 2\eta G^2 (\sqrt{n K} + 4 K ))
        + \frac{D^2}{ 2 \eta n K} + \frac{\eta G^2}{2}
        \\
        &\leq
        8 \eta G^2 (\sqrt {n K} + K)
        + \frac{D^2}{ 2 \eta n K} + \frac{\eta G^2}{2}
        \\
        &\leq \frac{6 G D}{n^{1/4} K^{1/4}} + \frac{4 K^{1/4}}{n^{3/4}}
        ,
    \end{align*}
    where the last inequality follows from a choice of $\eta = D/ (2 G n^{3/4}K^{3/4} )$.
    When $n \geq K$, the above implies
    \begin{align*}
        \E \sb{ F(\hat w) - F(w^\star) }
        \leq \frac{10 G D}{n^{1/4} K^{1/4}},
    \end{align*}
    and concludes the proof.
  \end{proof}

\section{Stability Lemmas}
In this section, we provide statements and proofs for several known results relating to stability properties of SGD. For convenience, we repeat the definition of UAS \cref{eq:uas_A} with notation suitable for SGD;
\begin{align}
    \epstabSGD (t) \eqq \max_{f_1, \ldots, f_t, f'} \max_{i\in[t]}
    \normb{w_{t+1} - w_{t+1}^{(i)}},
    \label{eq:sgd_uniform_stab}
\end{align}
where $f_1, \ldots, f_t, f'$ are any sequence of convex Lipschitz losses, $w_{t+1}$ the iterate produced by gradient descent from $w_1\in W$ on $f_1, \ldots, f_t$, and $w_{t+1}^{(i)}$ the iterate produced from $w_1$ on the same sequence after replacing $f_i$ with $f'$.

The next lemma relates the difference between the without-replacement loss distribution and the full batch objective to the uniform stability rate \cref{eq:sgd_uniform_stab} of the optimization algorithm in question. For a proof see \cite{sherman2021optimal} (where it was originally stated for average stability, which is a weaker notion and thus implies the uniform stability case as well).
\begin{lemma}
\label{lem:stab_gen}
    Let $\cb{f(w; t)}_{t=1}^{n}$ be a set of $n$, $G$-Lipschitz losses, and $F(w) = \frac{1}{n} \sum_{t=1}^n f(w ; t)$.
    Then, for a uniformly random permutation $\pi\colon [n] \leftrightarrow [n]$, and $w_1$ independent of $\pi$, it holds that
    \begin{align*}
        \E_{\pi} \sb{ F(w_t) - f(w_t; \pi(t)) }
        \leq \frac{(t-1)G}{n} \epstabSGD(t-1) ,
    \end{align*}
    where $\epstabSGD$ is the stability rate of SGD defined in \cref{eq:sgd_uniform_stab}, and $w_t$ the output of SGD on $\cb{f(w; s)}_{s=1}^{t-1}$.
\end{lemma}

% \begin{definition}
%     The stability rate of an algorithm $\A\colon Z^* \to \R^d$ is defined by
%     \begin{align*}
%         \epstab (t) \eqq \max_{z_1, \ldots, z_t, z' \in Z} \max_{i\in[t]}
%             \norm{\A(S) - \A(S^i)}
%     \end{align*}
% \end{definition}

Following are two lemmas providing uniform stability upper bounds for SGD.
\begin{lemma}
    \label{lem:sgd_stab}
    The uniform argument stability of SGD with step size $\eta > 0$ on convex $G$-Lipschitz losses is bounded as;
    \begin{align*}
        \epstabSGD (t) \leq 2 G \eta \sqrt t.
    \end{align*}
\end{lemma}
For proof of the above lemma, see \cite{bassily2020stability}.

Next, we have standard lemmas providing stability of ERM and regularized ERM in respectively strongly convex and general convex problems.
\begin{lemma}[Strongly Convex ERM Stability]\label{lem:sc_erm_stability}
    Let $f\colon W \times Z \to \R$ be $\lambda$-strongly convex and $G$-Lipschitz for all $z\in Z$.
    Then
    \begin{align*}
        \av{ \E_{S \sim \Z^n} \sb{ F(\wERM) - \hat F(\wERM)} }
        \leq \frac{G^2}{\lambda n}
    \end{align*}
\end{lemma}
\begin{proof}
    Let $\hat w_S \eqq \wERM$ denote the empirical risk minimizer, and 
    $\hat w_{S^{i}}$ the ERM for the training set with the $i$'th index swapped with a fresh sample $z_i'$.
    We have 
    \begin{align*}
        \av{ \E_{S \sim \Z^n} \sb{ F(\hat w_S) - \hF(\hat w_S)} }
        &= \av{ \frac{1}{n} \sum_{i=1}^n \E\sb{
            f(\hat w_S; z_i') - f(\hat w_{S^i}; z_i')
        }}
        \\
        &\leq \frac{G}{n} \sum_{i=1}^n \E
            \norm{\hat w_S - \hat w_{S^i}}
        \leq \frac{4 G^2}{\lambda n},
    \end{align*}
    where the first inequality is the generalization equals average stability (see e.g., \cite{shalev2010learnability}), and the last inequality follows since $\hat w_{S}$ and $\hat w_{S^i}$ minimize $(1/\lambda n)$-objectives that differ in a $2G$-Lipschitz term.
\end{proof}

\begin{lemma}[Regularized ERM Stability]\label{lem:rerm_stability}
    Let $f\colon W \times Z \to \R$ be $G$-Lipschitz for all $z\in Z$, and denote the regularized empirical risk minimizer by 
    \begin{aligni*}
        \hat w_S^\lambda \eqq \argmin_{w\in W}\cb{\hF(w) + \frac{\lambda}{2}\norm{w}^2}.
    \end{aligni*}
    Then
    \begin{align*}
        \av{ 
            \E_{S \sim \Z^n} \sb{ 
                F(\hat w_S^\lambda) - \hF(\hat w_S^\lambda)} 
        } \leq \frac{G^2}{\lambda n}
    \end{align*}
\end{lemma}
\begin{proof}
    Let $F^\lambda(w) \eqq F(w) + \frac{\lambda}{2}\norm{w}^2$ and define the regularized empirical loss $\hat F^\lambda$ accordingly. Then we have a $\lambda$-strongly convex problem and by  \cref{lem:sc_erm_stability},
    \begin{align*}
        \av{ \E_{S \sim \Z^n} \sb{ 
            F (\hat w_S^\lambda) - \hF (\hat w_S^\lambda)} 
        } = \av{ 
            \E_{S \sim \Z^n} \sb{ F^\lambda (\hat w_S^\lambda) - \hat F^\lambda (\hat w_S^\lambda)}
        } \leq \frac{4 G^2}{\lambda n}.
    \end{align*}
\end{proof}

\section{Auxiliary Lemmas}
The following provides standard step size dependent lower bounds for convex optimization. See also \cite{amir21sgd} where similar claims are made in their Lemma 6.2 and implicit in the proof of their Theorem 6.1.
\begin{lemma}\label{lem:std_convex_lb}
    For any step-size $\eta > 0$, $T \in \N$ and $d \eqq \lceil 16\eta^2T^2 \rceil$, there exists a convex optimization problem $h \colon W \to \R$ where $W\subseteq \R^{d+1}$ is of constant diameter such that
    \begin{align*}
        h(\hat w) - \min_{w\in W}h(w) \geq \frac{1}{8}\min\cb{ \frac{1}{\eta T} + \eta , 1},
    \end{align*}
    and $\hat w$ is any suffix average of $T$ gradient descent step iterates.
\end{lemma}
\begin{proof}
    We shall concatenate two objectives; the first is single dimensional and will contribute the $\eta$ term, the second is $d$ dimensional and will contribute the $1/\eta T$ term.
    \paragraph{First objective.}
    Set $f_1(w) \eqq \av{w - \eta/4}$. Since we initialize at $0$, the iterates will ``zig-zag'' between $0$ and $-\eta$. Clearly, any average of iterates is at best $\eta/4$ away from zero loss.
    \paragraph{Second objective.}
    Set 
    \begin{aligni*}
        f(w) \eqq \max_{i\in [d]} \cb{w(i)},
    \end{aligni*}
    and note $w^\star = -\frac{1}{\sqrt{d}} \boldsymbol{1}$ where $\boldsymbol{1}$ denotes the all ones vector $\in \R^d$.
    We initialize SGD at $w_1 = 0 \in \R^d$, and follow the gradient steps $\nabla f(w_t) = e_i$ where $i\in[d]$ is one of the coordinates that satisfy $w_t(i) \geq w_t(j) \; \forall j\in[d]$. Hence, for any $t\in[T]$,
    \begin{align*}
        \norm{w_{t+1}}_1 \leq 
        \norm{w_t}_1 + \eta \norm{\nabla f(w_t)}_1
        = \norm{w_t}_1 + \eta
        \leq \cdots \leq \eta t \leq \eta T.
    \end{align*}
    By the pigeonhole principle, this implies there must exist some coordinate $i$ such that $w_{T+1}(i) \geq -\eta T/d$.
    In addition, for any $i$, $\wbar w_{\tau:T}(i) \geq w_{T+1}(i)$.
    Therefore, assuming $8\eta^2T^2 \geq 1$ we conclude;
    \begin{align*}
        f(\wbar w_{\tau:T}) - f(w^\star) 
        \geq -\frac{\eta T}{d} + \frac{1}{\sqrt d}
        % \geq -\frac{\eta T}{16\eta^2T^2-1} + \frac{1}{4 \eta T}
        \geq -\frac{\eta T}{8\eta^2T^2} + \frac{1}{4 \eta T}
        = \frac{1}{8 \eta T}.
    \end{align*}
    In the case where $8\eta^2T^2 < 1$, 
    \begin{align*}
        f(\wbar w_{\tau:T}) - f(w^\star) 
        \geq -\eta T + \frac{1}{2}
        % \geq -\frac{\eta T}{4\eta^2T^2-1} + \frac{1}{2 \eta T}
        \geq \frac{1}{2} -\frac{1}{2\sqrt 2}
        \geq \frac{1}{4},
    \end{align*}
    and the result follows.
\end{proof}

\begin{lemma}\label{lem:loss_lb_util}
    Let $\cb{x_1, \ldots, x_n}$ be a set of real numbers. Then
    \begin{align*}
        \frac{1}{n}\sum_{s=1}^n \sqrt{\sum_{t=s+1}^n x_t^2} 
        \geq \frac{1}{5 \sqrt n} \sum_{t=n/4}^n \av{x_t}.
    \end{align*}
\end{lemma}
\begin{proof}
    \begin{align*}
        \frac{1}{n}\sum_{s=1}^n \sqrt{\sum_{t=s+1}^n x_t^2} 
        &\geq \frac{1}{n} \sum_{s=1}^n 
            \frac{1}{\sqrt{n- s}} \sum_{t=s+1}^n \av{x_t}
        \tag{by Jensen's inequality}
        \\
        &= \frac{1}{n} \sum_{t=1}^n \av{x_t}
            \sum_{s<t} \frac{1}{\sqrt{n- s}} 
        \\
        &= \frac{1}{n} \sum_{t=1}^n \av{x_t}
            \sum_{j=n-t+1}^{n-1} \frac{1}{\sqrt{j}} 
        \\
        &\geq \frac{1}{n} \sum_{t=1}^n \av{x_t}
            \int_{n-t+1}^n \frac{1}{\sqrt x} \mathrm d x
        \\
        &= \frac{2}{n} \sum_{t=1}^n \av{x_t} \b{\sqrt{n} - \sqrt{n-t+1}}
        \\
        &\geq \frac{2}{n} \sum_{t=n/4}^n \av{x_t}\b{\sqrt{n} - \sqrt{3n/4}}
        \\
        &\geq \frac{2}{\sqrt n} \sum_{t=n/4}^n \av{x_t}\b{1 - \sqrt{3/4}}
        \\
        &\geq \frac{1}{5 \sqrt n} \sum_{t=n/4}^n \av{x_t}.
    \end{align*}
\end{proof}
\end{document}